\documentclass[12pt]{l4dc2022}
\usepackage{enumerate}
\usepackage{color}
\usepackage{wrapfig}
\usepackage{algorithm}
\usepackage[noend]{algpseudocode}
\usepackage{algorithmicx}
\usepackage{pdflscape}

\usepackage{multirow}
\usepackage{makecell}

\newcommand{\gauss}[3]{\mathcal{N}\left( #1 \middle| #2, #3 \right)}
\newcommand{\KMinv}{ K_{MM}^{-1}}
\newcommand{\KMinvDelta}{\left( K^\Delta_{MM}\right)^{-1}}
\newcommand{\KMdelinv}{ \left(K_{MM}^\Delta \right)^{-1}}
\newcommand{\KdelT}[1]{ \left(K_{#1}^\Delta \right)^{\top}}
\newcommand{\mut}[1]{ \widetilde{\mu}_{#1}}
\newcommand{\St}[2]{ \widetilde{S}_{#1,#2}}
\newcommand{\invSt}[1]{ \widetilde{S}^{-1}_{#1,#1}}

\newcommand{\beginsupplement}{%
        \setcounter{table}{0}
        \renewcommand{\thetable}{S\arabic{table}}%
        \setcounter{figure}{0}
        \renewcommand{\thefigure}{S\arabic{figure}}%
     }

\newcommand{\doubletilde}[1]{\tilde{\tilde{#1}}}
\newcommand{\Stt}[2]{ \doubletilde{S}_{#1,#2}}

\newcommand{\doublehat}[1]{\hat{\hat{#1}}}

\title[Traversing Time with Multi-Resolution Gaussian Process State-Space Models]{Traversing Time with Multi-Resolution Gaussian Process State-Space Models}
\usepackage{times}

\author{\Name{Krista Longi}  \Email{krista.longi@helsinki.fi} \\
\addr University of Helsinki\\
\Name{Jakob Lindinger} \Email{jakob.lindinger@de.bosch.com} \\
\addr Bosch Center for Artificial Intelligence\\
\Name{Olaf Duennbier} \Email{Olaf.Duennbier@de.bosch.com} \\
\addr Robert Bosch GmbH\\
\Name{Melih Kandemir} \Email{kandemir@imada.sdu.dk} \\
\addr University of Southern Denmark\\
\Name{Arto Klami} \Email{arto.klami@helsinki.fi} \\
\addr University of Helsinki\\
\Name{Barbara Rakitsch} \Email{barbara.rakitsch@de.bosch.com} \\
\addr Bosch Center for Artificial Intelligence}

\begin{document}

\maketitle

\begin{abstract}
Gaussian Process state-space models capture complex temporal dependencies in a principled manner by placing a Gaussian Process prior on the transition function.
These models have a natural interpretation as discretized stochastic differential equations, but inference for long sequences with fast and slow transitions is difficult. Fast transitions need tight discretizations whereas slow transitions require backpropagating the gradients over long subtrajectories. 
We propose a novel Gaussian process state-space architecture composed of multiple components, each trained on a different resolution, to model effects on different timescales. The combined model allows traversing time on adaptive scales, providing efficient inference for arbitrarily long sequences with complex dynamics. 
We benchmark our novel method on semi-synthetic data and on an engine modeling task. 
In both experiments, our approach compares favorably against its state-of-the-art alternatives that operate on a single time-scale only.
\end{abstract}

\begin{keywords}%
  State-Space Model, Gaussian Process%
\end{keywords}


\section{Introduction}
\label{sec:intro}

Time-series modeling lies at the heart of many tasks: forecasting the daily number of
new cases in epidemiology~\citep{zimmer2020},  optimizing stock portfolios~\citep{heaton2017deep} or predicting the emissions of a car engine~\citep{yu2020deep}. 
In many cases, we do not know the underlying physical model but instead need to learn the dynamics from data, ideally in a non-parametric manner to support arbitrary dynamics. Irrespective of the total amount of data, many interesting phenomena (e.g. rapid transitions in dynamics) manifest only in a small subset of the samples, calling for probabilistic
forecasting techniques.

Gaussian Process state-space models (GPSSMs) hold the promise to model non-linear, unknown dynamics in a probabilistic manner by placing a Gaussian Process (GP) prior on the transition function~\citep{wang2005gaussian, frigola2015bayesian}.
While inference has been proven to be challenging for this model family, there has been a lot of progress in the past years and 
recent approaches vastly improved the scalability~\citep{eleftheriadis2017identification, doerr2018probabilistic}.

For long trajectories, methods updating the parameters using the complete sequence converge poorly due to the vanishing and exploding gradient problem~\citep{pascanu2013difficulty}.
While specialized architectures help circumventing the problem in the case of recurrent neural networks~\citep{hochreiter1997lstm, chung2014empirical}, it is not clear how one can apply these concepts to GP models. Furthermore, the problem of large runtime and memory footprints for training persists, as the gradients need to be backpropagated through the complete sequence.
A natural solution is to divide the trajectory into mini-batches which
reduces training time significantly, but also lowers the flexibility of the model: long-term effects that evolve slower than the size of one mini-batch can no longer be inferred~\citep{williams1995gradient}.

To address the problem of modeling long-term dependencies while retaining the computational advantage of mini-batching, we propose a novel GPSSM architecture with $L$ additive components.
The resulting posterior is intractable, and we apply variational inference to find an efficient and structured approximation~\citep{blei2017variational}. 
To capture effects on different time scales, our training scheme cycles through the components, whereby each component $l \in \{1,...,L\}$, 
is trained on a different resolution. 
For training the low-resolution components, 
we downsample the observations of the sequence, allowing us to pack a longer history in a mini-batch of fixed size (see Figure~\ref{fig:resolution}).

We further show that our algorithm is grounded in a coherent statistical framework by interpreting the GP transition model as a stochastic differential equation (SDE) similar to~\citet{hegde2018deep}. 
This relationship enables us to train our components on adaptive scales to capture effects on multiple time scales.
From a numerical perspective, our method decomposes the dynamics of the data according to its required step sizes. 
Long-term effects are learned by low-resolution components corresponding to large step sizes, short-term effects by high-resolution components corresponding to small step sizes.

We validate our new algorithm experimentally and show that it works well in practice on semi-synthetic data and on
a challenging engine modeling task.
Furthermore, we demonstrate that our
algorithm outperforms its competitors by a large margin in cases where the dataset consists of fast and slow dynamics.
For the engine modeling task, we introduce a new dataset to the community that contains the raw emissions of a gasoline car engine and has over 500,000 measurements. The dataset is available at \url{https://github.com/boschresearch/Bosch-Engine-Datasets}.

\section{Background on GPSSMs and SDEs}
\label{sec:background}
\paragraph{Gaussian Processes in a Nutshell} 
The GP prior, $f(x) \sim GP(0, k(x,x'))$, defines a distribution over functions, $f: \mathbb{R}^{D_x} \rightarrow \mathbb{R}$, and is fully specified 
by the kernel $k: \mathbb{R}^{D_x} \times \mathbb{R}^{D_x} \rightarrow \mathbb{R}$. Given a set of arbitrary inputs, $x_{M}=\{x_m\}_{m=1}^M$, their function values, $f_{M} = \{f(x_m)\}_{m=1}^M$, follow a Gaussian distribution 
$
p(f_{M}) = \mathcal{N}(f_{M} \vert 0 , \text{K}_{MM})
$
where   $\text{K}_{MM}=\{k(x_m, x_{m'})\}_{m=1, m'=1}^M$. 

For a new set of input points, $x_N = \{x_n\}_{n=1}^N$,
the predictive distribution over the corresponding function values, $f_N = \{f(x_n) \}_{n=1}^N$, can then be obtained
by conditioning the joint distribution on $f_M$,
leading to $p(f_N \vert f_M) = \mathcal{N}(f_N \vert \mu(x_N), \Sigma(x_N))$ with
\begin{eqnarray}
\mu (x_N) &=& 
\text{K}_{NM} \text{K}_{MM}^{-1} f_M , \label{eq:gpmean}\\
\Sigma (x_N) &=&
\text{K}_{NN} - \text{K}_{NM} \text{K}_{MM}^{-1} \text{K}_{NM}^{\top},\label{eq:gpcov}
\end{eqnarray}
where the cross-covariances $K_{NM}$ are defined similarly as $\text{K}_{MM}$, i.e.
$\text{K}_{NM} = \{k(x_n, x_m)\}_{n,m=1}^{N,M}$.
For a more detailed introduction, we refer the interested reader to~\citet{rasmussen2006gaussian}.

\paragraph{Gaussian Process State-Space Models}
\label{sec:gpssm}
We are given a dataset $y_T =\{y_t\}_{t=1}^T$ over $T$ time points. Each time point $t$ is characterized by the outputs $y_t \in \mathbb{R}^{D_y}$. 
State-space models (see e.g.~\citet{sarkka2013bayesian}) offer a general way to describe time-series data by introducing a latent state, $x_t \in \mathbb{R}^{D_x}$, that captures the compressed history of the system, for each time point $t \in \{1, \ldots, T\}$. 
Assuming the process and observational noise to be i.i.d.\ Gaussian distributed, the model can be written down as follows:
\begin{eqnarray*}
x_{t+1} \vert x_{t}\sim  \mathcal{N}(x_{t+1} \vert  x_{t} +  f(x_{t}),  \text{Q}), \;\;\;\;\;\;\;\;\;\;
y_t \vert x_t  \sim \mathcal{N}(y_t \vert g(x_t), \Omega), 
\end{eqnarray*}
where $f:\mathbb{R}^{D_x} \rightarrow \mathbb{R}^{D_x}$ models the change of the latent state in time and $g: \mathbb{R}^{D_x} \rightarrow \mathbb{R}^{D_y}$ maps the latent state to the observational space. 
The covariance $\text{Q} \in \mathbb{R}^{{D_x} \times {D_x}}$ describes the process noise, and
 $\Omega \in \mathbb{R}^{{D_y} \times {D_y}}$ the observational noise. 
Following the literature~\citep{wang2005gaussian, deisenroth2011pilco}, we assume that the update in the latent state can be modeled under a GP prior, i.e. $f(x) \sim GP(0, k(x, x'))$.\footnote{To be more precise, each latent dimension $d\in \{1,\ldots, D_x\}$ follows an independent GP prior. We suppressed the dependency on the latent dimension $d$ for the sake of better readability in our notation.}
The model generalizes easily to problems with exogenous inputs that we left out in favor of an uncluttered notation.

Finally, we chose a linear model $g(x_t) = Cx_t$ with output matrix $C \in \mathbb{R}^{D_y \times D_x}$ as emission function.
This is a widely adopted design choice, since the linear emission model reduces non-identifiabilities of the solution~\citep{frigola2015bayesian}.
Our approach generalizes to non-linear models as well, which might for instance be important in cases in which prior knowledge supports the use of more expressive emission models.

\paragraph{Sparse Parametric Gaussian Process State-Space Models}
\label{sec:fitc}
Sparse GPs augment the model by a set of inducing points $(x_M, f_M)$ that can be exploited during inference to summarize the training data in an efficient way.
~\citet{snelson2005sparse} introduced the so-called FITC (fully independent training conditional) approximation on the augmented joint density $p(f_M, f_N)$ by assuming independence between the function values, $f_N$, conditioned on the set of inducing points, $f_M$, i.e.
$p(f_N, f_M) \approx \prod_n  p(f_n \vert f_M) p(f_M)$ with $p(f_n \vert f_M) = \mathcal{N}(f_n \vert \mu(x_n), \Sigma(x_n))$.
Recently, this and similar formulations have regained interest in the community, since they simplify inference and yield good empirical performance~\citep{jankowiak2020parametric, rossi2020rethinking}.
We follow this line of work by assuming the same conditional factorization, 
\begin{align}
f_M &\sim \mathcal{N}(f_M \vert 0, \text{K}_{MM}), \label{eq:ind_prior} \\
f_t \vert f_M & \sim \mathcal{N}(f_t \vert  \mu(x_{t}),  \Sigma(x_t)), \label{eq:fitc0} \\
x_{t+1} \vert x_t, f_t &\sim  \mathcal{N}(x_{t+1} \vert  x_{t} + f_t, \text{Q}), \label{eq:fitc}
\end{align}
where $f_t$ are the GP predictions at time index $t$ with mean $\mu(x_t)$ and covariance $\Sigma(x_t)$ [Eqs.~\eqref{eq:gpmean} and \eqref{eq:gpcov}].
The FITC approximation has also found its way into the GPSSM literature:
\citet{doerr2018probabilistic} use it in the same way as we do (see also discussion in~\citet{ialongo2019overcoming}).

The difference to the standard formulation is mostly pronounced if we sample twice from the same region with large GP uncertainty [Eq.~\eqref{eq:gpcov}]: samples from the FITC prior are assumed to be independent, while samples from the standard prior are correlated.
Since the GP uncertainty is only large in input regions that are not covered by the inducing points, we deem the differences to be rather subtle and accept them in favor of establishing a bridge between the GPSSM and the GPSDE formulation, as we show in the following. 

\begin{figure*}[t]
    \centering
    \includegraphics[width=.9\textwidth]{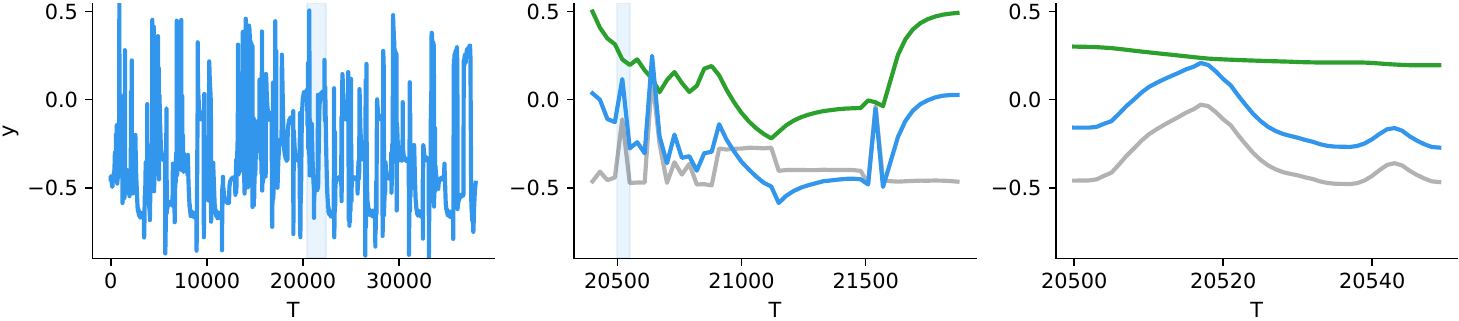}
    \caption{\textbf{Resolution Matters.}
    A semi-synthetic dataset (blue) created as a sum of two functions, one with fast varying dynamics (gray) and one with slowly varying dynamics (green)).
  \textit{Left:} Shown is the complete trajectory. Using all observations for each parameter update is too time- and memory-consuming.
  \textit{Middle:} Shown is a dilated mini-batch of size $50$ for which we have selected every $30$th observation. It allows for fitting the slowly varying function (green). 
  The fast dynamics (gray) cannot be inferred as the observations are too sparse. 
 \textit{Right:} Shown is a mini-batch of size $50$ on the standard resolution. 
  The mini-batch only covers a short interval of the trajectory as can be noted by the different time axis.
  It allows for fitting the fast varying function (gray).
  The slow dynamics (green) cannot be inferred as the gradient information is too weak.
  Our algorithm allows learning on multiple resolutions to capture effects on different timescales.
  }
    \label{fig:resolution}
\end{figure*}

\paragraph{Gaussian Process Stochastic Differential Equations}

SDEs can be regarded as a stochastic extension to ordinary differential equations where randomness enters the system via Brownian motion.
Their connection to GPSSMs is obtained by considering the SDE
\vspace{-9pt}
\begin{align}
\label{eq:sde}
dx_t = f(x_t)dt + \sqrt{\text{Q}^\Delta}dW_t,
\end{align}
where the drift term is given by the GP predictions $f(x_t) \sim \mathcal{N}(\mu^\Delta(x_t), \Sigma^\Delta(x_t))$ [Eqs.~\eqref{eq:gpmean} and \eqref{eq:gpcov}], the diffusion term by $\sqrt{\text{Q}^\Delta}$, and the Brownian motion by ${W}_t \in \mathbb{R}^{D_x}$.
In order to clearly distuingish the notation from the discrete GP transition function in Section~\ref{sec:gpssm}, 
we endow all potentially different quantities with a $\Delta$.
Applying a GP prior over the drift function has been done
previously 
 in~\citet{ruttor2013approximate} and \citet{yildiz2018learning}.
A related parameterization has also been suggested by~\citet{hegde2018deep} to extend deep GPs to an infinite number of hidden layers.

The solution to Eq.~\eqref{eq:sde} is a stochastic process over $x_t$.  
Except for a few cases, such as linear time-invariant systems, SDEs cannot be solved analytically and require numerical integration.
Hence, we apply the Euler-Maruyama scheme (see e.g.~\cite{sarkka2019applied}) to draw approximate samples:
\begin{align}
\label{eq:fj}
f_j \vert f_M \sim&\
\mathcal{N}(f_{j} \vert \mu^\Delta(x_j), \Sigma^\Delta(x_j)), \\
    \label{eq:em}
   x_{j+ 1} \vert x_j, f_j  
 \sim &\
    \mathcal{N}(x_{j+1} \vert x_j + R \Delta_t f_j, R \Delta_t \text{Q}^\Delta),
\end{align}
where $f_j$ corresponds to the GP prediction at index $j$.
The stepsize is given by $R \Delta_t$ where $R$ is the resolution and
 $\Delta_t$ is the time interval between two adjacent observations in the time series $y_T$.
Note that we employ the index $j$ to denote the time indices in the Euler-Maruyama scheme, whereas we use the index $t$ in the GPSSM formulation.
Consequently, a time index $t$ indicates a time $t \Delta_t$ after the starting time, whereas the index $j$ signifies a time $jR\Delta_t$ after the starting time.

The Euler-Maruyama method
converges to the true solution with shrinking step size $R \Delta_t$.
Prior work often sets $R  \ll 1$ corresponding to a deep GP transition function which is justified by its strong order of convergence of 1/2.
However, the convergence order only states that we need to increase the number of Monte Carlo samples quadratically in order to achieve a linear reduction in the expected approximation error.
We cannot deduce an appropriate step size $R \Delta_t$ from the convergence order alone, since the latter highly depends on the SDE form which is in our case characterized by the kernel, $k^\Delta(x,x')$, and the covariance $\text{Q}^\Delta$.


\section{Multi-Resolution Gaussian Process State-Space Models}
\label{sec:method}
Standard training of GPSSM models is restricted to a single resolution which hampers inference for long sequences with fast and slow transitions.
In this work, we introduce a novel GPSSM architecture that decomposes the latent space into multiple independent components.
We first extend doubly-stochastic variational inference for this model class.
Then, we show that this inference scheme can be generalized
such that each component is learned with a dedicated resolution in order to capture effects on different timescales.
Our training algorithm builds on  the observation that we can interpret the GPSSM transition function as a discretized SDE, which allows us to train each component with a different resolution under a unifying framework.

\subsection{Probabilistic Model}
\begin{wrapfigure}{R}{0.3\textwidth}
\centering
 \vspace{-20pt}
\includegraphics[width=0.2\textwidth]{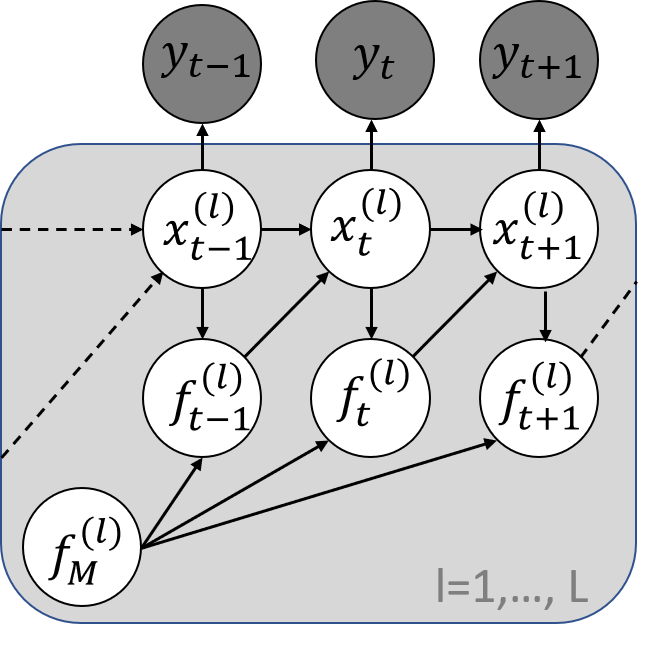}
    \caption{\textbf{Plate diagram.}
  }
    \label{fig:cartoon}
\end{wrapfigure}

\label{sec:model}
Our model splits the latent state into $L$ components,
 $x_t = \{x_t^{(l)} \}_{l=1}^L$,
that evolve independently over time:
\begin{align}
\label{eq:multi_prior}
f_t^{(l)} | f^{(l)}_M &\sim  \mathcal{N}(f^{(l)}_t \vert  \mu^{(l)}(x^{(l)}_{t}), \Sigma^{(l)}(x^{(l)}_t))\\
x^{(l)}_{t+1} \vert x^{(l)}_{t}, f_t^{(l)}
& \sim   
\mathcal{N}(x^{(l)}_{t+1} \vert  x^{(l)}_{t} +  f^{(l)}_{t} ,
\text{Q}^{(l)}). \label{eq:multi_transition} 
\end{align}

All terms are given by their equivalents in Eqs.~\eqref{eq:fitc0} and \eqref{eq:fitc} with 
 $x_t^{(l)} \in \mathbb{R}^{D_l}$ and $\sum_{l} D_l = D_x$.
Note that our proposed model can be cast into the standard formulation (Section~\ref{sec:fitc}) when allowing separate kernel hyperparameters for each latent state: the kernel hyperparameters are shared for all latent states within one component, and the latent component $x_t^{(l)}$ depends only on $x_{t-1}^{(l)}$ by the use of automatic relevance determination.
We chose a structured latent space in order to be able to learn each component with a different resolution (see Section \ref{sec:inference_multi}) which would not be possible within the standard framework.

\paragraph{Augmented model}
Collecting and simplifying all terms, we arrive at the augmented joint density
\begin{align}
p(x_0, x_{T}, f_M, y_T) 
=&
\prod_{l=1}^L
p(x^{(l)}_0)
\prod_{l=1}^L
p(f^{(l)}_M) 
\prod_{t=1}^T
p(y_t \vert x_t) 
 \prod_{t=0,l=1}^{T-1,L}
p(x^{(l)}_{t+1} \vert x^{(l)}_{t}, f^{(l)}_M) , 
\label{eq:pjoint}
\end{align}
where $x_0 = \{x_0^{(l)} \}_{l=1}^L$ are the initial latent states. 
We assume that their distribution decomposes between the components and $p(x_0^{(l)}) = \mathcal{N}(x^{(l)}_0 \vert \mu^{(l)}_0, \text{Q}^{(l)}_0)$ with mean $\mu^{(l)}_0 \in \mathbb{R}^{D_l}$ and covariance $\text{Q}^{(l)}_0 \in \mathbb{R}^{D_l \times D_l}$.

The transition probability $p(x^{(l)}_{t+1} \vert x^{(l)}_{t}, f^{(l)}_M) = \mathcal{N}(x^{(l)}_{t+1} \vert  x^{(l)}_{t} +  \mu^{(l)}(x^{(l)}_{t}),
\text{Q}^{(l)} + \Sigma^{(l)}(x^{(l)}_t))$ is obtained by marginalizing out the $f^{(l)}_t$ [which we assume to be conditionally independent given the $f^{(l)}_M$, see Eq.~\eqref{eq:multi_prior}] from Eq.~\eqref{eq:multi_transition} via standard Gaussian integrals.

While it is hard to read out from the formulas directly, analytically marginalizing out the inducing points $f_M^{(l)}$ from Eq.~\eqref{eq:pjoint} leads to a coupling between all latent states $x^{(l)}_T$ as we state in the following theorem.
\begin{theorem}
\label{thm_main1}
	For the prior of the GPSSM in Eq.~\eqref{eq:pjoint}, the marginals of the latent state at time point $ t\in \{1,\dots,T\} $, can be obtained as
	\begin{equation*}
		p(x^{(l)}_t) = \int p(x^{(l)}_0) \left[ \prod_{t'=1}^{t} p(x^{(l)}_{t'}|x^{(l)}_{0:t'-1}) \right] \prod_{t'=0}^{t-1} dx^{(l)}_{t'},
	\end{equation*}
	where all terms are Gaussian, $p(x^{(l)}_{t}|x^{(l)}_{0:t-1}) = \gauss{x^{(l)}_t}{\doublehat{\mu}_t^{(l)}}{\doublehat{\Sigma}_t^{(l)}}$, 
	and the mean $\doublehat{\mu}_t^{(l)}$ and covariance $\doublehat{\Sigma}_t^{(l)}$ depend on all previous states $x^{(l)}_{0:t-1} = \{x_{t'}^{(l)}\}_{t'=0}^{t-1}$.
\end{theorem}

This is a shortened version of Thm.~\ref{theoremprior} in Appx.~\ref{sec:appxprior}.
There, we provide the exact formulas for $p(x^{(l)}_t)$ and provide a proof which is based on induction.
While we do not use the marginal $p(x^{(l)}_t)$ in our inference scheme since it scales with $O(t^3)$, we use a slight generalization of this theorem in order to prove the equivalence between the GPSSM and the discretized SDE formulation (see Section \ref{sec:inference_multi}). 


In the following, we derive how this model type can be trained over a single resolution. 
We then proceed in showing that our model formulation allows training each component on a different resolution under a single objective by interpreting the GP transition function from a SDE perspective.

\subsection{Training over a Single Resolution}
\label{sec:inference}

Multi-component GPSSMs can be trained over a single resolution by extending the work of~\citet{doerr2018probabilistic}.
We start by introducing the structured approximate posterior
\begin{align}
q({x}_{0},x_T, {f}_{M})&=
\prod_{l=1}^L
q(x^{(l)}_0) 
\prod_{l=1}^L
q(f^{(l)}_M)
 \prod_{t=0,l=1}^{T-1,L}
p(x^{(l)}_{t+1} \vert x^{(l)}_{t}, f^{(l)}_M).
\label{eq:q}
\end{align}
The approximate posterior over the inducing outputs decomposes between the components and is given by
$q(f_M^{(l)})= \mathcal{N}\left({f}^{(l)}_{M} \mid m^{(l)}_{M}, \text{S}^{(l)}_{M}\right)$ with free parameters $m^{(l)}_M \in \mathbb{R}^{M}, \text{S}^{(l)}_M \in \mathbb{R}^{M \times M}$. 
We choose as variational distribution over the initial latent states  $q(x^{(l)}_0)= \mathcal{N}\left({x}^{(l)}_{0} \mid m^{(l)}_{0}, \text{S}^{(l)}_{0} \right)$, where $m^{(l)}_0 \in \mathbb{R}^{D_l}, \text{S}^{(l)}_0 \in \mathbb{R}^{D_l \times D_l}$ are free parameters. 
More flexible recognition models can easily be incorporated~\citep{doerr2018probabilistic}.

\paragraph{Variational Inference} 
\sloppy
We want to find the optimal values for the variational parameters 
that minimize the KL divergence between the approximate posterior $q(\cdot)$ and the true posterior $p(\cdot \vert y_T)$. 
Analogously, we can maximize the lower bound  $\mathcal{L}$ to the log marginal likelihood~\citep{blei2017variational}:
\begin{align}
\mathcal{L}=&\quad
\mathbb{E}_{ q(x_0, x_T, f_M)} \left[ 
\log \frac{p(x_0, x_T, f_M, y_T)}{q(x_0, x_T, f_M)} 
\right]
\label{eq:elbo1}
\\
=&
\sum_{t=1}^T \mathbb{E}_{q\left({x}_{t}\right)}\left[\log p\left({y}_{t} \mid {x}_{t}\right)\right]
- \operatorname{KL}\left(q\left(x_{0}\right) \big \Vert
p\left(x_{0} \right)\right)
-  \operatorname{KL}\left(q\left(f_{M}\right) \big \Vert p\left(f_{M} \right)\right),
\label{eq:elbo2}
\end{align}
where Eq.~\eqref{eq:elbo2} results from plugging Eqs.~\eqref{eq:pjoint} and~\eqref{eq:q} into Eq.~\eqref{eq:elbo1}.
Here $q(x_0)=\prod_{l=1}^L q(x^{(l)}_0)$ and analogously $p(x_0)$, $q(f_M)$ and $p(f_M)$.
The marginal $q(x_t)=\prod_{l=1}^L q(x^{(l)}_t)$  decomposes between the components with $q(x^{(l)}_t)= \int q(x^{(l)}_t \vert f_M^{(l)}) q(f^{(l)}_M) df_M^{(l)}$ and
$
 q(x^{(l)}_t \vert f_M^{(l)})
 =
  \int q(x^{(l)}_0)  \prod_{t'=0}^{t-1} p(x^{(l)}_{t'+1} \vert x^{(l)}_{t'}, f^{(l)}_M) \prod_{t'=0}^{t-1}  dx^{(l)}_{t'}
$.

As a final remark, the variational distribution $q(x^{(l)}_t)$ has no closed-form solution 
and we present different Monte Carlo sampling strategies in Supplementary Material~\ref{sec:appxgpssm}.
While all of these sampling schemes can be combined with our multi-resolution training, we adopt in our experiments the scheme from~\citet{ialongo2019overcoming} since it leads to unbiased samples and scales linearly with $O(t)$.

\paragraph{Backfitting Algorithm}
Since the variational posterior [Eq.~\eqref{eq:q}] decomposes between the components, we can apply an iterative learning algorithm for parameter optimization.
The backfitting algorithm~\citep{breiman1985estimating} cycles through all $L$ components to find the optimal set of parameters $\Theta = \{\theta^{(l)}\}_{l=1}^L$ where $\theta^{(l)} = \{m_0^{(l)}, S_0^{(l)}, m_M^{(l)}, S_M^{(l)} \}$.
In each step, we perform an inner optimization to update the parameters $\theta^{(l)}$ of the $l$-th component, while keeping all other parameters 
$\Theta \setminus \theta^{(l)}$ fixed.
While the benefits of a sequential learning scheme might not be clear yet, we will exploit its assumptions in the subsequent section to learn the parameters $\theta^{(l)}$ of each component with a different resolution in order to capture effects on multiple time scales. 

\paragraph{Mini-Batching}
Since the lower bound $\mathcal{L}$ decomposes between the time points, we can obtain an unbiased estimate using only a subset of the sequence~\citep{bottou2010large}, 
$\sum_{t=1}^T \mathbb{E}_{q\left({x}_{t}\right)}\left[\log p\left({y}_{t} \mid {x}_{t}\right)\right]  \approx \frac{T}{B}
 \sum_{t=t_0}^ {B+t_0} \mathbb{E}_{q\left({x}_{t  }\right)}\left[\log p\left({y}_{t } \mid {x}_{t}\right)\right]$
where $B$ is the batch size and $t_0$ denotes the first time index in the batch.
To sample efficiently from the marginal  
$q\left({x}_{t  }\right)$, we make one rather common approximation~\citep{aicher2019stochastic}:
We break the temporal dependency between $x_{t}$, and its predecessors $x_0, \ldots, x_{t_0-B_0}$, where $B_0$ is the buffer size,
by sampling $x_{t_0-B_0}$ directly from the recognition model, $q(x_0)$.
Together with the reparameterization trick~\citep{kingma2013auto}, we can exploit this subsampling scheme for computing cheap gradients during parameter optimization. 
However, breaking the temporal dependency also leads to biased gradients: effects that evolve slower than the size of the mini-batch can no longer be inferred. 

In principle, one could resolve this issue by 
downsampling the data in a preprocessing step.
However, this comes at the expense of fast varying dynamics that can then no longer be modeled (see Figure~\ref{fig:resolution}).
We compare to this approach in our experiments.

\subsection{Training over Multiple Resolutions}
\label{sec:inference_multi}
Prior work on GPSSMs takes only the dynamics of a single resolution into account which is not sufficient if effects on multiple time scales are present.
To circumvent this shortcoming, we proceed by interpreting the GP transition model through the lens of SDEs.

\paragraph{Relationship to SDEs}
Consider multi-component state-space models in which the transition model of the $l$-th component is given by 
\begin{align}
   p^\Delta(x^{(l)}_{j+ 1} \vert x^{(l)}_j, f^{(l)}_M)  =&
   \ \mathcal{N}(x^{(l)}_{j+1} \vert x^{(l)}_j + R \Delta_t \mu^\Delta(x^{(l)}_j), 
\label{eq:emm}
\quad 
(R \Delta_t)^2 \Sigma^\Delta(x^{(l)}_j) + R \Delta_t \text{Q}^\Delta),
\end{align}
where $\mu^\Delta(x^{(l)}_j)$ and $\Sigma^\Delta(x^{(l)}_j)$ are the equivalents of the GP mean and variance predictions in Eqs.~\eqref{eq:gpmean} and~\eqref{eq:gpcov},
and we again
 marginalized the local latent variables $f_j^{(l)}$ out of the discretized SDE [Eqs.~\eqref{eq:fj} and \eqref{eq:em}] using standard Gaussian calculus.
After restricting $R \geq 1$ to be integer,
we define all remaining terms of the model and the structured variational family analogously as in Eqs.~\eqref{eq:pjoint} and \eqref{eq:q}, leading to the lower bound
\begin{align}
\mathcal{L}_{\Delta}=&
\sum_{j=1}^J \mathbb{E}_{q^\Delta \left({x}_{j}\right)} \! \left[\log p^\Delta \! \left({y}_{j} \mid {x}_{j}\right)\right]
\! -
\!
 \operatorname{KL}\left(q^\Delta \! \left(x_{0}\right) \! \big \Vert
p^\Delta \! \left(x_{0} \right)\right) 
-
  \operatorname{KL}\left(q^\Delta \left(f_{M}\right) \big \Vert p^\Delta \left(f_{M} \right)\right),
\label{eq:elbo_sde}
\end{align}
where $J=T/R$. 
We next present the equivalence between the GP and discretized SDE formulations for $R=1$, i.e. for equal time steps.
\begin{theorem}
\label{thm_main2}
For $R=1$, there exists a setting of the model and variational parameters of the SDE formulation in terms of those of the GP formulation such that ~$\mathcal L=\mathcal{L}_\Delta$.
\end{theorem}
\begin{proof} We provide the exact parameterization and a constructive proof in Supplementary Material~\ref{sec:appxsde}.
\end{proof}

In our proof, we first provide the analytical formulae for the marginalization over the inducing outputs $f_M$ in the SDE and in the GPSSM formulation (similar to Theorem~\ref{thm_main1}). After showing that these formulae are consistent, we show that this consistency is passed on to the evidence lower bound.

Our findings allow us to reinterpret the GP transition model [Eq.~\eqref{eq:fitc}] as a discretized SDE with $R=1$. Choosing a resolution $R>1$, we can approximate
the GPSSM lower bound [Eq.~\eqref{eq:elbo2}] using the SDE formulation [Eq.~\eqref{eq:elbo_sde}].
In consequence, this relationship allows us to train with multiple resolutions by applying different approximation levels $R$.
In the following, we take this to our advantage in order to
come up with an efficient algorithm to learn effects on multiple time scales.

\paragraph{Multi-Resolution Learning}
Our algorithm decomposes the dynamics into $L$ components corresponding to different time scales. 
The components are fit iteratively by using the backfitting algorithm whereby each component is inferred with a different resolution.
For training the components of lower resolutions, we dilate the minibatch scheme by taking only every $R$-th observation into account in order to load larger histories into a mini-batch of fixed size $B$.
However, naively computing the marginal $q(x^{(l)}_{t_0+BR})$ would be too expensive since it requires $BR$ sampling steps.
We can overcome this issue by interpreting the component under the SDE perspective with resolution level $R$ using the lower bound [Eq.~\eqref{eq:elbo_sde}]
which allows us to
draw instead $B$ approximate samples from $q^\Delta(\cdot)$  [Eq.~\eqref{eq:emm}].
Hence, we can approximate the lower bound at different resolution levels with a fixed runtime,
while the 
approximation level of the marginal is adjusted to
the resolution level of the component under consideration.
Fast transitions are captured by high-resolution components with tight discretization levels ($R=1$), while slow transitions are captured by low-resolution components with long histories $(R >1)$. 
Since our variational family assumes that the latents are independent between components [Eq.~\eqref{eq:q}],
 we can compute the simulated latents of all but the $l$-th component, $x_T^{(\neq l)}$, outside of the inner optimization scheme of the backfitting algorithm.
The latter leads not only to a reduction in runtime, but also enables the use of different resolution levels across components in order to ensure that the discretization level is sufficiently tight for fast dynamics and the history length is sufficient long for slow dynamics.
We detail out the algorithm and provide its runtime analysis in Supplementary Material~\ref{sec:pseudocode}.
\paragraph{Limitations}
\label{sec:limitations}
We build on the variational family of~\cite{doerr2018probabilistic}, that uses the prior $p(x^{(l)}_t \vert x^{(l)}_{t-1}, f_M)$ as approximate smoothing distribution $q(x^{(l)}_t \vert \cdot)$. 
While extensions to more complex variational posteriors exist, they do not allow for mini-batching~\citep{ialongo2019overcoming} or make strong independence assumptions on $q(f^{(l)}_M, x^{(l)}_{1:T})$ \citep[e.g.][]{eleftheriadis2017identification}.
The methodological novelty of our work is to a large extent agnostic to the choice of  $q(x^{(l)}_t \vert \cdot)$
 and we expect that improvements on the inference scheme for general GPSSMs can be easily combined with our work.

\section{Experiments}
\label{sec:experiments}
We validate the presented algorithm on semi-synthetic data in Section~\ref{sec:toydata} and on an emission modeling task in Section~\ref{sec:engine}. 
Both experiments confirm that using multiple resolutions compares favorably to state-of-the-art methods that operate on a single resolution only.
We compare our novel multi-resolution GPSSM (MR-GPSSM) against the standard GPSSM applying a similar inference scheme ~\citep{doerr2018probabilistic}. 
To  tease apart the effects of multiple components and multiple resolutions, we additionally introduced the multi-component GPSSM (MC-GPSSM).
The latter has the same architecture and employs the same optimization algorithm as MR-GPSSM, but applies a single resolution over all components.
We refrained from benchmarking against other non state-space GP models since this has already been done extensively in~\cite{doerr2018probabilistic}, demonstrating the benefits of their method that we compare against.
We report the performance via the root mean squared error (RMSE) and the negative test log likelihood (nLL).
The latter evaluation metric as well as many more experimental details can be found in the Supplementary Material~\ref{sec:appxexperiments}. Code is available at \url{https://version.helsinki.fi/MUPI/mr-gpssm}.

\subsection{Semi-Synthetic Data}
\label{sec:toydata}
First, we benchmarked our method on $4$ semi-synthetic datasets (see Supplementary Material~\ref{sec:appx:semi_synth}) with varying properties: fast dynamics (F), mixed dynamics (M1, M2), and slow dynamics (S). 
Dataset M1 and M2 exhibit both fast and slow dynamics, and are challenging for previous methods.
All datasets are depicted in Supplementary Figure~\ref{figapp:simulated_data} and a close-up of dataset M1 is provided in Figure~\ref{fig:resolution}.
Each dataset consists of $T=37,961$ time points, from which we used the first half for training and the second half for testing.

\begin{table*}[t]
\centering
\caption{
\textbf{Results on Semi-Synthetic Data.}
Predictive performance of GPSSM variants on four semi-synthetic datasets with varying dynamics: slow (S), mixed (M1, M2) and fast (F). The experiment is repeated $5$ times and we report the mean (standard error) over all runs.
The best performing method, and all methods whose mean statistic overlap within the standard error, are marked in bold.}
\begin{small}
\begin{tabular}{lllllllll} 
\hline
\multicolumn{1}{l}{}  & \multicolumn{1}{l}{} & \multicolumn{2}{c}{\textbf{GPSSM }}            &  & \multicolumn{2}{c}{\textbf{MC-GPSSM} }  &  & \multicolumn{1}{c}{\textbf{MR-GPSSM} (ours)}  \\ 
\cline{3-4}\cline{6-7}\cline{9-9}
\multicolumn{1}{l}{}  & \multicolumn{1}{l}{} & $R=1$                  & $R=30$                &  & $R=[1,1]$              & $R=[30,30]$    &  & $R=[30, 1]$                            \\ 
\hline
\multirow{4}{*}{RMSE} & \textbf{F}           & \textbf{0.05 (0.00)}   & 0.14 (0.00)           &  & \textbf{0.06 (0.01)}   & 0.16 (0.01)    &  & 0.07 (0.01)                            \\
                      & \textbf{M1}          & 0.16 (0.02)            & 0.14 (0.00)           &  & 0.15 (0.01)            & 0.15 (0.00)    &  & \textbf{0.08~(0.01)}                   \\
                      & \textbf{M2}          & 0.14 (0.00)            & 0.29 (0.11)           &  & 0.14 (0.00)            & 0.20 (0.01)    &  & \textbf{0.09 (0.01)}                   \\
                      & \textbf{S}           & 0.33 (0.08)            & \textbf{0.16 (0.01)}  &  & 0.29 (0.02)            & 0.20 (0.03)    &  & \textbf{0.17 (0.02)}                   \\ 
\hline
\end{tabular}
\end{small}
\label{table:synthetic}
\end{table*}

\begin{figure*}[t]
    \centering
    \includegraphics[width=\textwidth]{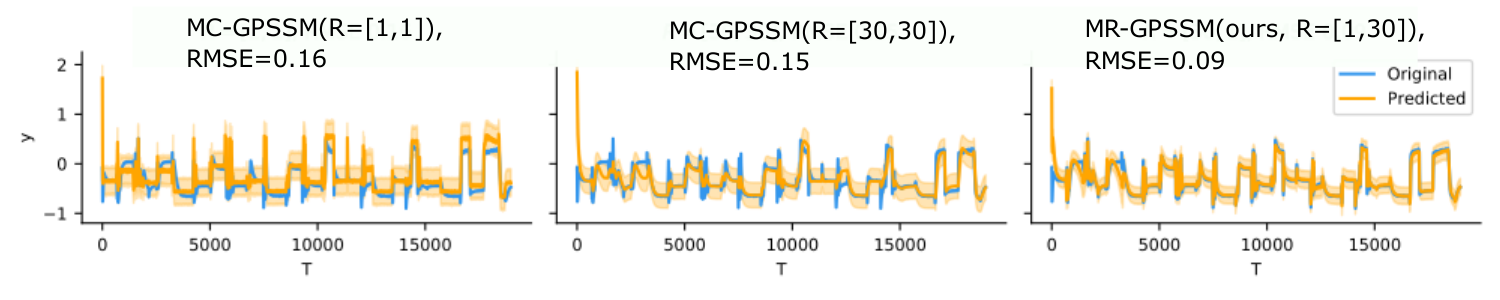}
    \caption{\textbf{Predictions on semi-synthetic dataset with mixed dynamics (M1).} 
    From left to right:
    MC-GPSSM ($R=[1,1]$, $R=[30,30]$), and MR-GPSSM ($R=[1,30]$).
    Our proposed model, MR-GPSSM, outperforms its competitors by capturing effects on different time scales, whereas MC-GPSSM ($R=1$) cannot model slow trends accurately and MC-GPSSM ($R=30$) does not catch all peaks.
  }
    \label{fig:simulated_oneexample}
\end{figure*}

For MR-GPSSM, we applied $L=2$ components with $D_x=2$ latent dimensions each, and learned one component with $R^{(f)}=1$ for fast dynamics and one with $R^{(s)}=30$ for slow dynamics. We trained each component for $600$ iterations that were split evenly into $12$ backfitting cycles.
We compared our model to MC-GPSSM using exactly the same settings.
For standard GPSSM, we set the number of latent states to $D_x=4$ and trained for $600$ iterations such that the model complexity and the number of parameter updates is comparable.
We varied the resolution for both comparison partners in $R \in \{1, 30\}$.
The results are shown in Table~\ref{table:synthetic} and Supplementary Table~\ref{table:synthetic2}. 
We observe that (MC-)GPSSM performs well if the resolution is chosen appropriately:
Fast dynamics (dataset F) can only be accurately predicted using a small resolution ($R=1$), whereas slow dynamics (dataset S) require a large resolution ($R=30$). 
Moreover, choosing the wrong resolution leads not only to a decrease in performance, but also to convergence problems which lead to the removal of  one run of GPSSM ($R=1$) on dataset S.
Our proposed model, MR-GPSSM, achieves comparable results on both tasks.
On datasets with mixed dynamics (M1, M2), MR-GPSSM improves over the single resolution models, since it is the only method that captures effects on multiple timescales (see Figure~\ref{fig:simulated_oneexample} for dataset M1 and Supplementary Figure~\ref{figapp:simulated_example} for the remaining datasets).

Next, we investigated if increasing the mini-batch size can provide an alternative solution for capturing slow dynamics. Instead of learning the dynamics with resolution $R=30$ and minibatch size $B=50$, as done previously, we increased the mini-batch size to $B=1500$ and applied the standard resolution $R=1$.
We confirm on dataset S that the latter strategy does not yield competitive results  even if we allow for prolonged training time (see Supplementary Table \ref{table:batchsizeVSresolution}).

\begin{figure*}[h]
    \centering
    \includegraphics[width=\textwidth]{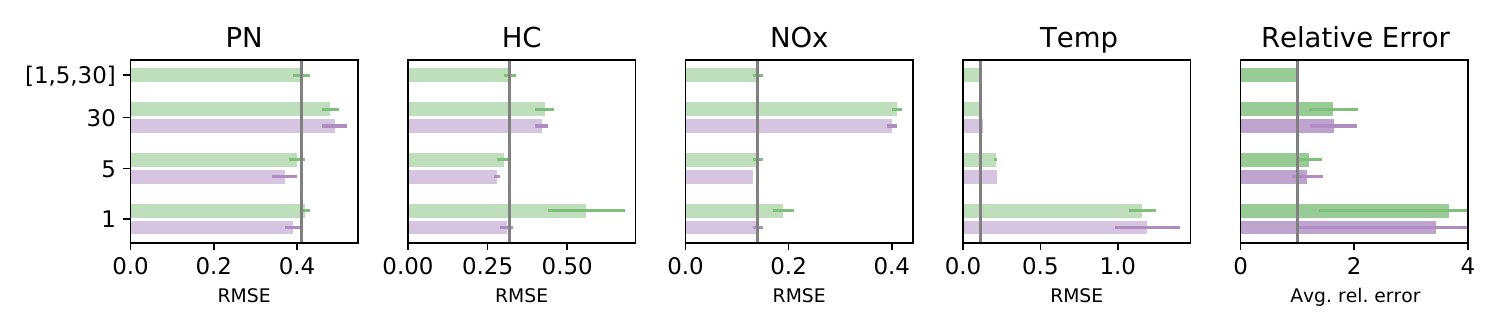}
    \caption{\textbf{Predictive Performance on Engine Modeling Task.} From left to right: RMSE on the four outputs PN, HC, NOx, Temp and relative error with respect to our method, averaged over all outputs.
    From top to bottom: We compare our method, MR-GPSSM (indicated with a gray line), to MC-GPSSM and GPSSM sorted according to decreasing resolution. MR/MC-GPSSM results are shown in green and GPSSM results in purple. Our method, MR-GPSSM, is the only method that performs consistently well over all outputs (see rightmost plot).
    We report the mean value and its standard error over 5 repetitions.}
    \label{fig:engine}
\end{figure*}
\subsection{Engine Modeling Task}
\label{sec:engine}
This dataset consists of 22 independent measurements containing the raw emissions of an engine.
Each measurement is recorded with $10$Hz and between $21$ and $63$ minutes long, resulting in over 500,000 data points. The system is described by $4$ inputs and the following $4$ outputs: particle numbers (PN), hydrocarbon concentration (HC), nitrogen oxide concentration (NOx) and engine temperature (Temp).
In the following, we split the data into 16 train and 6 test measurements.
For each output, the experiment is carried out $5$ times using stratified cross-validation since the design of experiment differs between measurements (see Supplementary Material~\ref{sec:appx_cross_validation}).
To avoid local optima, we repeated each training $3$ times using random restarts, and selected the model with the best training objective for predicting on the test set.

First, we studied if the optimal resolution differs between outputs by performing a grid search over $R \in \{1,5,10,20,30,40,50,60,70\}$ using standard GPSSM. 
We set the number of latent dimensions to $D_x=6$ and $3,000$ training iterations.
The results are depicted in Figure~\ref{fig:engine}
and Supplementary Figure~\ref{figapp:gridsearch}.
We observe that slow dynamics are in particular predominant for the output Temp.
Next, we trained MR/MC-GPSSM using a comparable configuration ($L=3$, $D_x=2$, $3,000$ iterations per component). We set the resolutions of MR-GPSSM to $R=[1,5,30]$ such that the best resolution for each output is included, and trained MC-GPSSM on each resolution independently.
The results are shown in Figure \ref{fig:engine} and Supplementary Table~\ref{table:emission}.
We observe that the PN test error is higher than for the other outputs which is in line with literature that reports high noise level for the 
PN measurement system, e.g.~\cite{frommater2018phenomenological}. Our method, MR-GPSSM, shows competitive performance across all outputs, while (MC-)GPSSM works only well if
the resolution is set adequately. 
In addition, MR-GPSSM requires less fine-tuning, and also performs well if the resolution set is varied (see Supplementary Table \ref{table:sensitivity}).
\section{Conclusion}
We have presented a novel Gaussian Process state-space model architecture that allows to traverse time with multiple resolutions.
It is composed of multiple components that evolve independently over time. 
By interpreting the transition functions as discretized stochastic differential equations, we can learn each component with a different resolution to model effects on different time scales.

The benefits of our approach are demonstrated on semi-synthetic data and on a challenging engine modeling task.
However, our methodological contribution is general and can also be applied to use cases from different domains ranging from neuroscience~\citep{prince2021parallel}, medicine~\citep{lipton2015learning} to human motion prediction~\citep{martinez2017human}.

\section*{Contributions}

The initial problem setting and idea were given by Barbara Rakitsch, while the details and experiments were designed by all authors jointly. Krista Longi was responsible of implementing the algorithm as well as performing the experiments and analysing the results. Jakob Lindinger derived the theoretical analysis of the method and was responsible for its write-up. Olaf Dünnbier provided the data set and domain knowledge for engine modeling. Barbara Rakitsch wrote the main paper with contributions of all authors.

\bibliography{references}

\appendix

\providecommand{\upGamma}{\Gamma}
\providecommand{\uppi}{\pi}
\newpage
\onecolumn
\clearpage

\beginsupplement
\appendix

\setlength\parindent{0pt}
\begin{center}
{\Large{Supplementary Material for}}

\bigskip

{\LARGE{\textbf{Traversing Time with Multi-Resolution}}}

\smallskip

{\LARGE{\textbf{Gaussian Process State-Space Models}}}

\end{center}

\section{Sampling Schemes for Gaussian Process State-Space Models}
\label{sec:appxgpssm}
\subsection{Problem Statement}
In this chapter, we present multiple Monte Carlo schemes to obtain samples from the variational posterior $ q $ of Gaussian Process (GP) state-space models (SSMs).
More precisely, we are interested in the marginals of the latent state at time point $ t\in \{1,\dots,T\} $, $ q(x_t) $.
They can be obtained as
\begin{equation}\label{eq:marginal}
q(x_t) = \int q(x_0) q(f_M,x_t,\dots,x_1|x_0) df_M \prod_{t'=0}^{t-1}dx_{t'},
\end{equation}
where
\begin{align}\label{eq:joint}
q(f_M,x_t,\dots,x_1|x_0) &= q(f_M) \prod_{t'=0}^{t-1} p(x_{t'+1}|x_{t'},f_M),\\
\label{eq:qx0}
q(x_0) &= \gauss{x_0}{m_0}{S_0},  \\
\label{eq:qfM}
q(f_M) &= \gauss{f_M}{m_M}{S_M},\\
\label{eq:transition}
p(x_{t+1}|x_{t},f_M) &= \gauss{x_{t+1}}{x_{t} + K_{tM}\KMinv f_M}{Q+K_{tt}-K_{tM}\KMinv K_{tM}^\top }. 
\end{align}
Here, $ m_0 $, $ S_0 $, $ m_M $, and $ S_M $ are variational parameters,
while $ Q $ is a model parameter, all of which have to be inferred.
Furthermore, we have $ K_{tt}=k(x_t,x_t) $, $ K_{tM} = \{ k(x_t,x_m)\}_{m=1}^M $, and $ K_{MM} = \{ k(x_m,x_{m'})\}_{m,m'=1}^M $,
where $ k(\cdot,\cdot): \mathbb{R}^{D_x}\times \mathbb{R}^{D_x} \to \mathbb{R} $ is a kernel or covariance function and $ \{x_m\}_{m=1}^M $ is a set of $ M $ so-called inducing points.\footnote{Note
that in our notation indices $ m, M $ always indicate quantities related to inducing points,
while the indices $ t, T $ always indicate other observed or latent quantities,
such that e.g.\ $ x_{t=1} $ and $ x_{m=1} $ are generally not the same.}
In the following we consider a one-dimensional latent space, i.e.~$ D_x = 1 $ and $ x_m, x_t \in \mathbb{R} $, for a less cluttered notation and therefore a better readability.
An extension to other dimensionalities $ D_x $ and to $L$ additive components is straightforward. The results and proofs are very similar.

\subsection{Overview}
\label{sec:appxgpssmOverview}
In the following, we present different Monte Carlo schemes to obtain samples from the variational posterior $q(x_t)$. 

\paragraph{Full Monte Carlo Treatment~\citep{ialongo2019overcoming}.}
We obtain samples $x_t \sim q(x_t)$ according to Eq.~\eqref{eq:marginal} by first sampling the initial latent state $x_0 \sim q(x_0)$ and the inducing outputs $f_M \sim q(f_M)$, and subsequently the latent states $x_1, \ldots x_t$ using the conditional $p(x_{t+1} \vert x_t, f_M)$. 
Note that the same set of inducing outputs $f_M$ is used in each time step to acquire valid samples.
The sampling scheme scales with $O(t)$ and we employ this one in our experiments.

\paragraph{PR-SSM Sampling Scheme~\citep{doerr2018probabilistic}.}
The marginal distribution [Eq.~\eqref{eq:marginal}] is approximated by
\begin{align}\label{eq:prism}
q(x_t) &= 
\int q(x_0) \prod_{t'=0}^{t-1} \left( \int p(x_{t'+1}|x_{t'},f_M) q(f_M) df_M \right) dx_{t'}
\\
&=
\label{eq:prism_marg}
\int q(x_0)
 \prod_{t'=0}^{t-1}
  \mathcal{N} \big ( {x_{t'+1}} \vert {x_{t'} + K_{t'M}\KMinv m_M},  K_{t't'} - K_{t'M} \KMinv K_{t'M}^\top
 \\
&
 \quad \quad \quad  \quad \quad \quad \quad \quad \quad \quad \ \
  + K_{t'M}\KMinv S_M \KMinv K_{t'M}^\top \newline
+ Q 
    \big)
 dx_{t'}
 ,
\end{align}
which allows us to analytically marginalize out the inducing outputs $f_M$ [Eq.~\ref{eq:prism_marg}], while keeping the runtime requirement to $O(t)$.
However, this approximation also leads to a biased estimate of the lower bound $\mathcal{L}$ [Eq.~\ref{eq:elbo2}] as further discussed in~\citet{ialongo2019overcoming}.

\paragraph{Analytical Marginalization over the Inducing Outputs.}
We can analytically marginalize the inducing outputs $f_M$ from the variational posterior [Eq.~\eqref{eq:marginal}] as 
we summarize in the following theorem:

\begin{theorem} \label{theorem}
	For the variational posterior of the GP SSM as defined above, the marginals of the latent state at time point $ t\in \{1,\dots,T\} $, can be obtained as
	\begin{equation}\label{eq:thmmain}
		q(x_t) = \int q(x_0) \left[ \prod_{t'=1}^{t} q(x_{t'}|x_{t'-1},\dots,x_0) \right] \prod_{t'=0}^{t-1} dx_{t'},
	\end{equation}
	where all terms are Gaussian:
	\begin{align} \label{eq:thmgauss}
	q(x_{t}|x_{t-1},\dots,x_0) &= \gauss{x_t}{\hat{\mu}_t}{\hat{\Sigma}_t},\\
	\label{eq:thmmean}
	\hat{\mu}_t &= \mut{t-1} + \St{t-1}{0:t-2} \invSt{0:t-2} \left( x_{1:t-1} - \mut{0:t-2} \right),\\
	\label{eq:thmvar}
	\hat{\Sigma}_t &= \St{t-1}{t-1} - \St{t-1}{0:t-2} \invSt{0:t-2} \St{0:t-2}{t-1}.
	\end{align}
	Here, the terms are given by
	\begin{align}
	\label{eq:thmmu}
	\mut{t} &= x_t + K_{tM}\KMinv m_M,\\
	\label{eq:thmS}
	\St{t}{t'} &= K_{tM}\KMinv S_M \KMinv K_{t'M}^\top + \delta_{tt'}(Q + K_{tt} - K_{tM} \KMinv K_{t'M}^\top).
	\end{align}
\end{theorem}

The notation $ \cdot : \cdot $ is used to denote (block) column vectors or submatrices, e.g.~$ x_{1:t} = \begin{pmatrix} x_1 & \cdots & x_t \end{pmatrix}^\top \in \mathbb{R}^{t}$, or $\St{t-1}{0:t-2} = \begin{pmatrix} \St{t-1}{0} & \cdots & \St{t-1}{t-2} \end{pmatrix}  \in \mathbb{R}^{t-1} $.
Furthermore, $ \delta_{tt'} $ symbolizes the Kronecker delta.
Note that for $ t=1 $ the slices in the additional terms of Eqs.~\eqref{eq:thmmean} and \eqref{eq:thmvar} are empty and that therefore $ \hat{\mu}_1 = \mut{0} $ and $ \hat{\Sigma}_1 = \St{0}{0} $.

The problem as well as the idea for the proof is very similar to the one studied in Thm.~1 of~\citet{lindinger2020}:
The joint distribution of the latent states $x_1, ... , x_t$ cannot be seen as one joint multivariate Gaussian distribution since the latent state $x_{t'-1}$ enters the mean $\hat{\mu}_t$ and the covariance matrix $\hat{\Sigma}_t$ via the kernel matrices.
Instead, we need to come up with a recurrent formulation of the problem that is amenable to a proof by induction (see  Section~\ref{sec:appxproof} for more details).

Note that it is often the case that the analytical marginalization of global latent variables, here the $f_M$, can lead to faster convergence as opposed to marginalization via Monte Carlo sampling (see e.g.~\citet{lindinger2020},~\citet{kingma2015variational}). In this particular case, it is unlikely that the faster convergence (in terms of iterations) also results in a faster runtime (wall-clock time), as the cost of each iteration is increased from $O(t)$ to $O(t^3)$ when the analytical marginalization scheme from Thm.~\ref{theorem} is employed.
However, as we demonstrate in Sec.~\ref{sec:appxsde}, our result is of highly theoretical interest to the community: We exploit it to show that the marginal $q(x_t)$ [Eq.~\eqref{eq:marginal}] can be reinterpreted as a discretized SDE.


\subsection{Marginalization Proof}
\label{sec:appxproof}

In order to prove Thm.~\ref{theorem}, we require the following technical lemma:

\begin{lemma} \label{lemma}
	The term $ q(f_M,x_t,\dots,x_1|x_0) $ in Eq.~\eqref{eq:joint} can also be written as
	\begin{equation}\label{eq:lemmain}
	q(f_M,x_t,\dots,x_1|x_0) = q(f_M|x_t,\dots,x_0) \prod_{t'=1}^{t} q(x_{t'}|x_{t'-1},\dots,x_0),
	\end{equation}
	for $ t\in \{1,\dots,T\} $. Here, the $ q(x_{t}|x_{t-1},\dots,x_0) $ are as in Eqs.~\eqref{eq:thmgauss}-\eqref{eq:thmvar} and
	\begin{align}\label{eq:lemqfM}
	q(f_M|x_t,\dots,x_0) &= \gauss{f_M}{\hat{\mu}^t_M}{\hat{\Sigma}^t_M},\\
	\label{eq:lemmean}
	\hat{\mu}^t_M &= m_M + S_M \KMinv \left(K_{0:t-1,M}\right)^\top \invSt{0:t-1} \left( x_{1:t} - \mut{0:t-1} \right),\\
	\label{eq:lemvar}
	\hat{\Sigma}^t_M &= S_M - S_M \KMinv \left(K_{0:t-1,M}\right)^\top \invSt{0:t-1} K_{0:t-1,M} \KMinv S_M.
	\end{align}	
\end{lemma}
With a slight abuse of the slicing notation, we denote $ \left(K_{0:t-1,M}\right)^\top = \begin{pmatrix} K_{0,M}^\top & \cdots & K_{t-1,M}^\top \end{pmatrix}  \in \mathbb{R}^{M \times t}$. 
Before providing the proof of Lem.~\ref{lemma}, we show how this lemma can be used to prove Thm.~\ref{theorem}:

\begin{proof}[\emph{\textbf{Proof of Theorem \ref{theorem}}}]
	\label{thmproof}
	Starting with the definition of $ q(x_t) $ given in Eq.~\eqref{eq:marginal},
	\begin{equation}\label{eq:tproof1}
	q(x_t) = \int q(x_0) q(f_M,x_t,\dots,x_1|x_0) df_M \prod_{t'=0}^{t-1}dx_{t'},
	\end{equation}
	our aim is to show that this can equivalently be written as in Thm.~\ref{theorem}.
	Using Lem.~\ref{lemma}, more specifically Eq.~\eqref{eq:lemmain}, yields
	\begin{equation}\label{eq:tproof2}
	q(x_t) = \left(\int q(f_M|x_t,\dots,x_0) df_M \right) \int q(x_0)  \prod_{t'=1}^{t} q(x_{t'}|x_{t'-1},\dots,x_0) \prod_{t'=0}^{t-1}dx_{t'},
	\end{equation}
	where we pulled the only term depending on $ f_M $ out of the integral. As $ q(f_M|x_t,\dots,x_0) $ is a properly normalized probability density, the first integral equals one. This already completes the proof as the terms $ q(x_{t'}|x_{t'-1},\dots,x_0) $, according to Lem.~\ref{lemma}, have the correct form [Eqs.~\eqref{eq:thmgauss}-\eqref{eq:thmvar}].
\end{proof}

For the proof of Lem.~\ref{lemma} we will need two additional results.
The first result is about affine transformations of multivariate Gaussians: Given two Gaussian distributed variables $ x $ and $ y $ that obey
\begin{equation}\label{eq:gaussaffbefore}
p(x|y) = \gauss{x}{a + Fy}{A},\qquad \text{and} \qquad 
p(y) = \gauss{y}{b}{B},
\end{equation}
the following formulas hold (see e.g.~\citet{schoen2011} for a proof):
\begin{align}
\label{eq:gaussaffafter1}
p(x) &= \gauss{x}{a + Fb}{A +FBF^\top}, \\
\label{eq:gaussaffafter2}
p(y|x) &= \gauss{y}{b + BF^\top \left(A +FBF^\top \right)^{-1} \left[x-(a+Fb)\right]}{B- BF^\top \left(A +FBF^\top \right)^{-1}FB }.
\end{align}
Note that since
\begin{equation}\label{eq:gaussaffjoint}
p(x|y) p(y) = p(x,y) = p(y|x) p(x),
\end{equation}
Eqs.~\eqref{eq:gaussaffafter1} and \eqref{eq:gaussaffafter2} are particularly useful if we wish to rewrite the product of two Gaussian densities that are as in Eq.~\eqref{eq:gaussaffbefore}.
The second result is a well known formula for block matrix inversion:
\begin{equation}
\label{eq:blockinv}
\begin{pmatrix}
A & B \\
C & D
\end{pmatrix}^{-1} = 
\begin{pmatrix}
A^{-1} + A^{-1}B\widetilde{D}^{-1}CA^{-1} & -A^{-1}B\widetilde{D}^{-1} \\
-\widetilde{D}^{-1}CA^{-1} & \widetilde{D}^{-1}
\end{pmatrix},
\end{equation}
where $\widetilde{D} = D - CA^{-1}B$.

\begin{proof}[\emph{\textbf{Proof of Lemma \ref{lemma}}}]
	In the following we will prove the lemma by induction:
\paragraph{Base case}%
	We need to show that Eq.~\eqref{eq:lemmain} holds for $ t=1 $, i.e., that
	\begin{equation}\label{eq:indbase}
	q(f_M,x_1|x_0) = q(f_M|x_1,x_0) q(x_{1}|x_0)
	\end{equation}
	with the terms on the RHS given by Eqs.~\eqref{eq:lemqfM}-\eqref{eq:lemvar} and Eqs.~\eqref{eq:thmgauss}-\eqref{eq:thmvar}, respectively.
	
	In order to do so, we will perform the following steps:
	\begin{enumerate}[i)]
		\item In the first step we will show that Eqs.~\eqref{eq:gaussaffbefore}-\eqref{eq:gaussaffjoint} are applicable, which will enable us to write Eq.~\eqref{eq:indbase} as
		\begin{equation}\label{eq:indbasestep1}
		q(f_M,x_1|x_0) = \bar{q}(f_M|x_1,x_0) \bar{q}(x_{1}|x_0).
		\end{equation}
		\item Next, we will show that 
		\begin{equation}\label{eq:indbasestep2}
		\bar{q}(x_{1}|x_0) = q(x_{1}|x_0).
		\end{equation}
		\item In the final step, we will show that
		\begin{equation}\label{eq:indbasestep3}
		\bar{q}(f_M|x_1,x_0) = q(f_M|x_1,x_0).
		\end{equation}
	\end{enumerate}

	For step i), we start with the definition of $ q(f_M,x_1|x_0) $ in Eq.~\eqref{eq:joint}:
	\begin{align}
	q(f_M,x_1|x_0) &= q(f_M) p(x_{1}|x_0,f_M)\\ \label{eq:indbase1}
	&= \gauss{f_M}{m_M}{S_M} \gauss{x_{1}}{x_{0} + K_{0M}\KMinv f_M}{Q+K_{00}-K_{0M}\KMinv K_{0M}^\top },
	\end{align}
	where we used Eqs.~\eqref{eq:qfM} and \eqref{eq:transition} in the second step.
	Next, we note that the requirements in Eq.~\eqref{eq:gaussaffbefore} are given for the terms above, where we identify $ f_M $ as $ y $ and $ x_1 $ as $ x $.
	Applying Eqs.~\eqref{eq:gaussaffafter1}-\eqref{eq:gaussaffjoint} to Eq.~\eqref{eq:indbase1}, results in
	\begin{equation}\label{eq:indbase2}
	q(f_M,x_1|x_0) = \bar{q}(f_M|x_1,x_0) \bar{q}(x_{1}|x_0)
	\end{equation}
	with yet to be determined means and covariances, which concludes the first step.
	
	In step ii), we start with the second term on the RHS in Eq.~\eqref{eq:indbase2} and use Eqs.~\eqref{eq:gaussaffafter1} and \eqref{eq:indbase1}, yielding
	\begin{align}\label{eq:indbase3}
	\bar{q}(x_{1}|x_0) &=\mathcal{N} \big( {x_1} \vert {x_{0} + K_{0M}\KMinv m_M},
\nonumber \\
&  \quad  \quad  \quad  \quad \
{Q+K_{00}-K_{0m}\KMinv K_{0M}^\top + K_{0m}\KMinv S_M \KMinv K_{0M}^\top} \big) \\
	\label{eq:indbase4}
	&= \gauss{x_1}{\mut{0}}{\St{0}{0}} = \gauss{x_1}{\hat{\mu}_1}{\hat{\Sigma}_1},
	\end{align}
	where we used the definitions in Eqs.~\eqref{eq:thmmean}-\eqref{eq:thmS} in the last line.
	Together with the definition in Eq.~\eqref{eq:thmgauss}, this implies that in fact $ \bar{q}(x_{1}|x_0) = q(x_{1}|x_0)$.
	
	Finally for step iii), using Eqs.~\eqref{eq:gaussaffafter2} and \eqref{eq:indbase1} on the first term on the RHS in Eq.~\eqref{eq:indbase2} results in 
	\begin{align}
	\bar{q}(f_M|x_1,x_0) &=\mathcal{N} \big( { f_M } \vert { m_M + S_M \KMinv K_{0M}^\top \invSt{0}\left(x_1-\mut{0}\right) }, 
\\
&   \quad  \quad  \quad  \quad \ \
{ S_M - S_M \KMinv K_{0M}^\top \invSt{0} K_{0_M} \KMinv S_M } \big) \\
	&= \gauss{f_M}{\hat{\mu}^1_M}{\hat{\Sigma}^1_M}.	
	\end{align}
	In the first line we used that $ a + Fb = \mut{0} $ and that $ A+FBF^\top = \St{0}{0} $ (by comparing Eqs.~\eqref{eq:indbase3} and \eqref{eq:indbase4} with Eq.~\eqref{eq:gaussaffafter1}). Additionally we used the definitions in Eqs.~\eqref{eq:lemmean} and \eqref{eq:lemvar} in the last line.
	Together with the definition in Eq.~\eqref{eq:lemqfM}, this implies that in fact $ \bar{q}(f_M|x_1,x_0) = q(f_M|x_1,x_0)$, concluding step iii) and therefore also the base case of the induction.
	
\paragraph{Inductive step}
	We assume that Lem.~\ref{lemma} holds for some $ t=1,\dots,T-1 $ (induction assumption) and then need to show that it also holds for $ t+1 $.
	That is, assuming that
	\begin{equation}\label{eq:steprepeat}
	q(f_M,x_t,\dots,x_1|x_0) = q(f_M|x_t,\dots,x_0) \prod_{t'=1}^{t} q(x_{t'}|x_{t'-1},\dots,x_0),
	\end{equation}
	holds for some $ t $ with the terms on the RHS given by Eqs.~\eqref{eq:lemqfM}-\eqref{eq:lemvar}, and Eqs.~\eqref{eq:thmgauss}-\eqref{eq:thmvar}, respectively, we need to show that this implies that
	\begin{equation} \label{eq:stepmain}
	q(f_M,x_{t+1},\dots,x_1|x_0) = q(f_M|x_{t+1},\dots,x_0) \prod_{t'=1}^{t+1} q(x_{t'}|x_{t'-1},\dots,x_0),
	\end{equation}
	where the terms are again given by Eqs.~\eqref{eq:lemqfM}-\eqref{eq:lemvar} (but with $ t\to t+1 $), and Eqs.~\eqref{eq:thmgauss}-\eqref{eq:thmvar}, respectively.
	
	The way to show this is very similar to the way we showed the base case, the resulting formulas will only look more complicated and we will need one additional step in the beginning:
	\begin{enumerate}[i)]
		\item[o)] Starting with the LHS of Eq.~\eqref{eq:stepmain} and its definition in Eq.~\eqref{eq:joint}, we can regroup the terms as follows:
		\begin{align}
		q(f_M,x_{t+1},\dots,x_1|x_0) &= q(f_M) \prod_{t'=1}^{t+1} p(x_{t'}|x_{t'-1},f_M) \\
&= p(x_{t+1}|x_{t},f_M) \left(q(f_M) \prod_{t'=1}^{t} p(x_{t'}|x_{t'-1},f_M) \right)\\
		\label{eq:stepo1}
		&= p(x_{t+1}|x_{t},f_M) q(f_M,x_{t},\dots,x_1|x_0),
		\end{align}
		where we identified the terms from Eq.~\eqref{eq:joint} in the last step.
		We can therefore immediately apply the induction assumption [Eq.~\eqref{eq:steprepeat}] to the second term in Eq.~\eqref{eq:stepo1}, resulting in
		\begin{equation}\label{eq:stepo2}
		q(f_M,x_{t+1},\dots,x_1|x_0) = p(x_{t+1}|x_{t},f_M) q(f_M|x_t,\dots,x_0) \prod_{t'=1}^{t} q(x_{t'}|x_{t'-1},\dots,x_0),
		\end{equation}
		where the $ q(x_{t}|x_{t-1},\dots,x_0) $ terms are given by Eqs.~\eqref{eq:thmgauss}-\eqref{eq:thmvar}.
		Comparing this to what we want to show [Eq.~\eqref{eq:stepmain}], we see that it remains to be shown that
		\begin{equation}\label{eq:stepo}
		p(x_{t+1}|x_{t},f_M) q(f_M|x_t,\dots,x_0) = q(f_M|x_{t+1},\dots,x_0)  q(x_{t+1}|x_{t},\dots,x_0),
		\end{equation}
		such that the terms on the RHS are given by Eqs.~\eqref{eq:lemqfM}-\eqref{eq:lemvar}, and Eqs.~\eqref{eq:thmgauss}-\eqref{eq:thmvar}, respectively.
		From this point on, we will have to do the exact same steps as in the base case, which we will repeat below with updated indices.
	\end{enumerate}
	\begin{enumerate}[i)]
		\item In the first step we will show that Eqs.~\eqref{eq:gaussaffbefore}-\eqref{eq:gaussaffjoint} are applicable, which will enable us to write the LHS of Eq.~\eqref{eq:stepo} as
		\begin{equation}\label{eq:stepi}
		p(x_{t+1}|x_{t},f_M) q(f_M|x_t,\dots,x_0) = \bar{q}(f_M|x_{t+1},\dots,x_0)  \bar{q}(x_{t+1}|x_{t},\dots,x_0).
		\end{equation}
		\item Next, we will show that 
		\begin{equation}\label{eq:stepii}
		\bar{q}(x_{t+1}|x_{t},\dots,x_0) = q(x_{t+1}|x_{t},\dots,x_0).
		\end{equation}
		\item In the final step, we will show that
		\begin{equation}\label{eq:stepiii}
		\bar{q}(f_M|x_{t+1},\dots,x_0) = q(f_M|x_{t+1},\dots,x_0).
		\end{equation}
	\end{enumerate}

	Let us start with step i): We can use the definition in Eq.~\eqref{eq:transition} and Eqs.~\eqref{eq:lemqfM}-\eqref{eq:lemvar} (as part of the induction assumption) to write the terms on the LHS of Eq.~\eqref{eq:stepi} as
	\begin{align}\label{eq:stepi1p}
	p(x_{t+1}|x_{t},f_M)  &= \gauss{x_{t+1}}{x_{t} + K_{tM}\KMinv f_M}{Q+K_{tt}-K_{tM}\KMinv K_{tM}^\top } \\ \label{eq:stepi1q}
	q(f_M|x_t,\dots,x_0) &= \begin{aligned}[t]
	\mathcal{N} \left(f_M \middle| \vphantom{\invSt{0:t-1}} \right. &m_M + S_M \KMinv \left(K_{0:t-1,M}\right)^\top \invSt{0:t-1} \left( x_{1:t} - \mut{0:t-1} \right), \\
	&S_M - S_M \KMinv \left(K_{0:t-1,M}\right)^\top \invSt{0:t-1} K_{0:t-1,M} \KMinv S_M \left. \vphantom{\invSt{0:t-1}} \right)
	\end{aligned}
	\end{align}
	Next, we note that the requirements in Eq.~\eqref{eq:gaussaffbefore} are given for the terms above, where we identify $ f_M $ as $ y $ and $ x_{t+1} $ as $ x $.
	Applying Eqs.~\eqref{eq:gaussaffafter1}-\eqref{eq:gaussaffjoint} to Eqs.~\eqref{eq:stepi1p} and \eqref{eq:stepi1q}, allows us to write the LHS of Eq.~\eqref{eq:stepi} as 
	\begin{equation}\label{eq:stepi2}
	p(x_{t+1}|x_{t},f_M) q(f_M|x_t,\dots,x_0) = \bar{q}(f_M|x_{t+1},\dots,x_0)  \bar{q}(x_{t+1}|x_{t},\dots,x_0)
	\end{equation}
	with yet to be determined means and covariances. This concludes the first step.
	
	For step ii), we examine the second term on the RHS of Eq.~\eqref{eq:stepi2}, which can be obtained using Eqs.~\eqref{eq:gaussaffbefore} and \eqref{eq:gaussaffafter1} applied to Eqs.~\eqref{eq:stepi1p} and \eqref{eq:stepi1q}:
	\begin{equation}\label{eq:stepii1}
	\bar{q}(x_{t+1}|x_{t},\dots,x_0) = \gauss{ x_{t+1} }{ \bar{m}_{t+1} }{ \bar{\Sigma}_{t+1} }.
	\end{equation}
	The mean is given by
	\begin{align}
	\bar{m}_{t+1} &= x_t +  K_{tM}\KMinv \left[m_M + S_M \KMinv \left(K_{0:t-1,M}\right)^\top \invSt{0:t-1} \left( x_{1:t} - \mut{0:t-1} \right)\right] \\
	&= \mut{t} + \St{t}{0:t-1} \invSt{0:t-1} \left(x_{1:t} - \mut{0:t-1}\right) = \hat{\mu}_{t+1},
	\end{align}
	where we used the definitions in Eqs.~\eqref{eq:thmmean}, \eqref{eq:thmmu}, and \eqref{eq:thmS} in the second line. The covariance is given by 
	\begin{align}
	\bar{\Sigma}_{t+1} &= Q+K_{tt}-K_{tM}\KMinv K_{tM}^\top   \\
	& \quad +  K_{tM}\KMinv \left[S_M - S_M \KMinv \left(K_{0:t-1,M}\right)^\top \invSt{0:t-1} K_{0:t-1,M} \KMinv S_M\right] \KMinv K_{tM}^\top \\
	\label{eq:stepii2}
	&= \St{t}{t} - \St{t}{0:t-1} \invSt{0:t-1} \St{0:t-1}{t}  = \hat{\Sigma}_{t+1},
	\end{align}
	where we used the definitions in Eqs.~\eqref{eq:thmvar}-\eqref{eq:thmS} in the last line.
	Taken together, Eqs.~\eqref{eq:stepii1}-\eqref{eq:stepii2} state that $ \bar{q}(x_{t+1}|x_{t},\dots,x_0) $ is a Gaussian
	with mean $ \hat{\mu}_{t+1} $ and covariance $ \hat{\Sigma}_{t+1} $,
	i.e., that $ \bar{q}(x_{t+1}|x_{t},\dots,x_0) = q(x_{t+1}|x_{t},\dots,x_0) $.
	This concludes the second step.
	
	For the last step, step iii), we consider the first term on the RHS of Eq.~\eqref{eq:stepi2}, which can be obtained using Eqs.~\eqref{eq:gaussaffbefore} and \eqref{eq:gaussaffafter2} applied to Eqs.~\eqref{eq:stepi1p} and \eqref{eq:stepi1q}:
	\begin{equation}\label{eq:stepiii1}
	\bar{q}(f_M|x_{t+1},\dots,x_0) = \gauss{ f_M }{ \bar{\mu}^{t+1}_M }{ \bar{\Sigma}^{t+1}_M }.
	\end{equation}
	As in the previous step, it remains to be shown that the mean and covariance coincide with $ \hat{\mu}^{t+1}_M $ and $ \hat{\Sigma}^{t+1}_M $ given in Eq.~\eqref{eq:lemmean} and \eqref{eq:lemvar}, respectively.
	Showing this requires exactly the same steps (although with different quantities) as in Eqs.~(63)-(68) in the appendix of~\citet{lindinger2020}, so we will only sketch the derivation exemplarily for $ \bar{\mu}^{t+1}_M $ here:
	Starting from
	\begin{equation}\label{eq:stepiii2}
	\bar{\mu}^{t+1}_M = \begin{aligned}[t]
	&m_M + S_M \KMinv \left(K_{0:t-1,M}\right)^\top \invSt{0:t-1} \left( x_{1:t} - \mut{0:t-1} \right) +\\
	&\left[S_M - S_M \KMinv \left(K_{0:t-1,M}\right)^\top \invSt{0:t-1} K_{0:t-1,M} \KMinv S_M\right] \times \\
	&\KMinv K_{tM}^\top \hat{\Sigma}_{t+1}^{-1} \left[ x_{t+1} -\mut{t} - \St{t}{0:t-1} \invSt{0:t-1} \left(x_{1:t} - \mut{0:t-1}\right) \right],
	\end{aligned}  
	\end{equation}
	where $ \hat{\Sigma}_{t+1}^{-1} $ is as in Eq.~\eqref{eq:stepii2}, we can reorder the terms until we get to a point where we can apply the block matrix inversion formula [Eq.~\eqref{eq:blockinv}] backwards.
	In order to do so, we identify the terms of Eq.~\eqref{eq:blockinv} as
	\begin{equation}
	A = \St{0:t-1}{0:t-1}, \qquad B = \St{0:t-1}{t}, \qquad C = \St{t}{0:t-1}, \qquad
	\widetilde{D} = \hat{\Sigma}_{t+1}^{-1},
	\end{equation}
	and therefore $ D = \St{t}{t} $ (by comparing the definition of $ \widetilde{D} $ after Eq.~\eqref{eq:blockinv} with  $ \hat{\Sigma}_{t+1}^{-1} $ in Eq.~\eqref{eq:stepii2}).
	This results (after several steps) in
	\begin{align}
	\bar{\mu}^{t+1}_M &= m_M + S_M \KMinv \begin{pmatrix}
	\left(K_{0:t-1,M}\right)^\top & K_{tM}^\top
	\end{pmatrix}\begin{pmatrix}
	\St{0:t-1}{0:t-1} & \St{0:t-1}{t} \\[2mm]
	\St{t}{0:t-1} & \St{t}{t}
	\end{pmatrix}^{-1}
	\begin{pmatrix}
	x_{1:t} - \mut{0:t-1} \\[2mm] x_{t+1} -\mut{t}
	\end{pmatrix} \\
	&= m_M + S_M \KMinv \left(K_{0:t,M}\right)^\top \invSt{0:t}
	\left( x_{1:t+1} - \mut{0:t} \right) = \hat{\mu}_M^{t+1},
	\end{align}
	where we first rewrote the first line according to our slicing notation and then applied the definition in Eq.~\eqref{eq:lemmean} in the last step.
	Doing the same for the covariance, we can similarly show that $ \bar{\Sigma}^{t+1}_M = \hat{\Sigma}_M^{t+1} $.
	
	Hence $ \bar{q}(f_M|x_{t+1},\dots,x_0) = q(f_M|x_{t+1},\dots,x_0) $.
	This concludes step iii) as well as the induction step and therefore the proof of Lem.~\ref{lemma}.
	
\end{proof}

\newpage

\section{Relationship between Gaussian Process State-Space Models and Stochastic Differential Equations}
\label{sec:appxsde}
\subsection{Problem Statement}
In the following, we revisit the analytical marginalization of the inducing outputs for the GPSSM using the stochastic differential equation (SDE) formulation thereof.
We assume that the frequency of the original time series (corresponding to time steps $ \Delta_t $) is sufficiently small such that one Euler–Maruyama step between $R$ observations is enough.

Our main goal in this chapter is to find a setting of the variational and model parameters 
such that the marginals of the latent states that are needed to maximize the evidence lower bound are consistent between the SSM and the SDE formulation when choosing $R=1$.
As a by-product, we obtain analytical formulae for the marginals of the latent state for the general case ($R \geq 1$).

\subsection{Overview}
\label{sec:appxsdeOverview}
We first find an analytical formula for the marginalization over the inducing outputs $f_M$ in the SDE formulation of our problem.
Then, we proceed by showing that 
these formulae are consistent with the ones of the GPSSM formulation that we obtained in the previous session (supplementary material~\ref{sec:appxgpssm}). 
This consistency is further passed on to the evidence lower bound as we show in the final part of this Section.
All proofs are given in Section~\ref{appx:sde_proofs}.

Furthermore, we are
interested in this paper in the special case of the SDE formulation where we only consider constant step sizes
 $ R\Delta_t $, where $ R \geq 1 $ is an integer. 
In order to clearly distinguish the notation from the previous problem in Sec.~\ref{sec:appxgpssm}, we mark all (potentially different) quantities with a $ \Delta $.
Additionally, we use an index $ j $ to denote the time indices: Whereas before (in Sec.~\ref{sec:appxgpssm}), a time index $ t $ indicated a time $ t\Delta_t $ after the starting time, the index $ j $ signifies a time $ jR\Delta_t $ after the starting time.

As starting point, we study the marginals of the latent state of the variational posterior, $ q^\Delta (x_j) $, for all time indices $ j \in {1,\dots, J} $, where $ J = T/R$. Repeating Eqs.~\eqref{eq:marginal}-\eqref{eq:transition} for the SDE formulation, the marginal can be obtained as
\begin{equation}\label{eq:sde_marginal}
q^\Delta (x_j) = \int q^\Delta (x_0) q^\Delta (f_M,x_j,\dots,x_1|x_0) df_M \prod_{j'=0}^{j-1}dx_{j'},
\end{equation}
where
\begin{align}\label{eq:sde_joint}
q^\Delta (f_M,x_j,\dots,x_1|x_0)
&= q^\Delta (f_M) \prod_{j'=0}^{j-1} p^\Delta (x_{j'+1}|x_{j'},f_M),\\
\label{eq:sde_qx0}
q^\Delta (x_0) &= \gauss{x_0}{m_0^\Delta }{S_0^\Delta },\\
\label{eq:sde_qfM}
q^\Delta (f_M) &= \gauss{f_M}{m_M^\Delta }{S_M^\Delta },\\
\label{eq:sde_transition}
p^\Delta (x_{j+1}|x_{j},f_M) &=
 \mathcal{N} \Big( { x_{j+1} } \Big| {x_{j}+ R\Delta_t K_{jM}^\Delta \KMdelinv f_M},
\nonumber \\
& \quad \quad \quad \quad \quad  \
{ R\Delta_t Q^\Delta + (R \Delta_t)^2 \left[ K^\Delta_{jj} - K^\Delta_{jM}\KMdelinv \KdelT{jM} \right] } \Big),
\end{align}
where we have obtained Eq.~\eqref{eq:sde_transition} by analytical marginalization of the local latent variable $f_j$:
\begin{align}
\label{eq:sde_marg_fj}
p^\Delta (x_{j+1} \vert x_{j}, f_M)&= \int p^\Delta (x_{j+1}|x_{j},f_j)  p^\Delta(f_j \vert f_M) df_j, \quad \text{with}  \\
p^\Delta(f_j \vert f_M) &=  \mathcal{N}(f_j \vert 
 K_{jM}^\Delta \KMdelinv f_M, 
K^\Delta_{jj} - K^\Delta_{jM}\KMdelinv \KdelT{jM}) \\
\label{eq:sde_trans_fj}
p^\Delta(x_{j+1} \vert x_j, f_j) &=  \mathcal{N}(x_{j+1} \vert x_j +R \Delta_t f_j,R \Delta_t Q^\Delta). 
\end{align}
As before, $ m_0^\Delta $, $ S_0^\Delta $, $ m_M^\Delta $, and $ S_M^\Delta $ are variational parameters,
and $ Q^\Delta $ is a model parameter. We additionally defined $ K^\Delta _{jj} = k^\Delta(x_j,x_j) $, and similarly for $ K_{jM}^\Delta $ and $ K_{MM}^\Delta $, where $ k^\Delta(\cdot ,\cdot) $ is a kernel or covariance function.
The main difference of the two formulations can be seen in Eqs.~\eqref{eq:transition} and \eqref{eq:sde_transition}, where the latter has an additional dependency on $ R\Delta_t $ in the mean and the variance.
In the following, we provide the analytical formulas for the marginalization of the inducing outputs $ f_M $ from Eq.~\eqref{eq:sde_marginal}:
\begin{theorem} \label{sde_theorem}
	Using the variational posterior of the SDE formulation of the GP SSM as defined in Eqs.~\eqref{eq:sde_marginal}-\eqref{eq:sde_transition}
the marginals of the latent state at time point $ j\in \{1,\dots,J\} $, can be obtained as
	\begin{equation}\label{eq:sde_thmmain}
	q^\Delta(x_j) = \int q^\Delta(x_0) \left[ \prod_{j'=1}^{j} q^\Delta(x_{j'}|x_{j'-1},\dots,x_0) \right] \prod_{j'=0}^{j-1} dx_{j'},
	\end{equation}
	where all terms are Gaussian:
	\begin{align} \label{eq:sde_thmgauss}
	q^\Delta(x_{j}|x_{j-1},\dots,x_0) &= \gauss{x_j}{\hat{\mu}_j}{\hat{\Sigma}_j},\\
	\label{eq:sde_thmmean}
	\hat{\mu}_j &= \mut{j-1} + \St{j-1}{0:j-2} \invSt{0:j-2} \left( x_{1:j-1} - \mut{0:j-2} \right),\\
	\label{eq:sde_thmvar}
	\hat{\Sigma}_j &= \St{j-1}{j-1} - \St{j-1}{0:j-2} \invSt{0:j-2} \St{0:j-2}{j-1}.
	\end{align}
	Here, the terms are given by
	\begin{align}
	\label{eq:sde_thmmu}
	\mut{j} &= x_j + R \Delta_t K^\Delta_{jM}\KMinvDelta m^\Delta_M,\\
\label{eq:sde_thmS}
	\St{j}{j'} &=(R \Delta_t)^2 K^\Delta_{jM}\KMinvDelta S^\Delta_M \KMinvDelta {K^\Delta_{j'M}}^\top
	\nonumber \\
	& \quad + \delta_{jj'}  \left[R \Delta_t Q^\Delta + (R \Delta_t)^2 \left( K^\Delta_{jj} - K^\Delta_{jM} \KMinvDelta {K^\Delta_{j'M}}^\top \right)\right].
	\end{align}
\end{theorem}
This theorem gives a general formula for the marginals $ q^\Delta(x_j) $.
We use this general result to connect the state space model and the SDE formulation of our problem.

We proceed by searching for a setting of the variational and model parameters such that these analytical marginals are consistent with the result in Eq.~\eqref{eq:thmmain}.
By this we mean that if the time steps in both approaches are equal (which is the case if we set $ R=1 $) we want the marginals of the latent state at the same time indices to be equal, i.e.
\begin{equation}\label{eq:consistency}
R=1 \land j = t \Rightarrow q^\Delta(x_j) = q(x_t).
\end{equation}
The relation between the parameters of the SDE and the standard formulation necessary to achieve the consistency in Eq.~\eqref{eq:consistency} are
provided in following corollary.

\begin{corollary}
 \label{eqi_marg_cor}
	Using the variational posterior of the SDE formulation of the GP SSM as defined in Eqs.~\eqref{eq:sde_marginal}-\eqref{eq:sde_transition} and setting,
	\begin{gather}\label{eq:sde_thmsetting}
	m_0^\Delta = m_0, \quad S_0^\Delta=S_0, \quad  m_M^\Delta =  m_M, \quad S_M^\Delta = S_M , \quad Q^\Delta = Q/(R \Delta_t), \\
	\label{eq:k1}
	 k^\Delta(x_m ,x_{m'})= k(x_m ,x_m'), \quad \quad \quad \quad
	  k^\Delta(x_j ,x_{j'})= k(x_j ,x_{j'})/(R \Delta_t)^2 \\ 
	  \label{eq:k2}
	 k^\Delta(x_m ,x_{j})= k(x_m , x_j)/(R \Delta_t), \quad 
	 k^\Delta(x_j ,x_{m})= k(x_j , x_m)/(R \Delta_t),  
	\end{gather} 
the marginals in Eqs.~\eqref{eq:thmmain} and \eqref{eq:sde_thmmain} are equal for $R=1$ and $j=t$.
\end{corollary}
Note that the rescaling of the kernel function $k(\cdot, \cdot)$ in Eqs.~\eqref{eq:k1} and~\eqref{eq:k2}, is often done in multi-output learning, or more generally, can also be interpreted as a simple form of interdomain Gaussian Processes~\citep{gredilla2009}.

The evidence lower bound of the discretized stochastic differential equation can be computed as 
\begin{eqnarray}
\label{eq:sde_elbo1}
\mathcal{L}_{\Delta}
= - \text{KL}(q^\Delta(x_0) \Vert p^\Delta(x_0)) - \text{KL}( q^\Delta(f_M) \Vert p^\Delta(f_M)) + \sum_{j=1}^J  \mathbb{E}_{q^\Delta(x_j)} \left[ \log p^\Delta(y_j \vert x_j) \right].
\end{eqnarray}
with 
\begin{eqnarray}
 p^\Delta(x_0) &=& \mathcal{N}(x_0 \vert \mu^\Delta_0, Q_0^\Delta) 
 \\
 p^\Delta(f_M) &=&  \mathcal{N}(f_M \vert 0, K^{\Delta}_{MM}) 
 \\
  p^\Delta(y_j \vert x_j) &=&  \mathcal{N}(y_j \vert  C^\Delta x_j, \Omega^\Delta),
\end{eqnarray}
where $\mu_0^\Delta, Q_0^\Delta, C^\Delta, \Omega^\Delta$ are additional model parameters, and the 
marginals of the variational posterior $q^\Delta(x_j)$ are defined in Eq.~\eqref{eq:sde_thmmain}.
This can be contrasted to the evidence lower bound of the state-space model formulation, that is given by
\begin{eqnarray}
\label{eq:gpssm_elbo}
\mathcal{L}
= - \text{KL}(q(x_0) \Vert p(x_0)) - \text{KL}( q(f_M) \Vert p(f_M)) +
\sum_{t=1}^T  \mathbb{E}_{q(x_t)} \left[ \log p(y_t \vert x_t) \right],
\end{eqnarray}
where the variational factors are given by Eqs.~\eqref{eq:qx0},~\eqref{eq:qfM},~\eqref{eq:thmmain}. 
Similarly as in Corollary~\ref{eqi_marg_cor}, we are again interested in the setting of the variational and model parameters such that the lower bounds in Eq.~\eqref{eq:sde_elbo1} and Eq.~\eqref{eq:gpssm_elbo}
match, i.e. $\mathcal{L}= \mathcal{L}_{\Delta}$.

The neccessary conditions are given in our final corollary.
\begin{corollary}
 \label{eqi_elbo_cor}
 Setting
\begin{eqnarray}
\label{eq:sde_elbosetting}
\mu_0^\Delta=\mu_0, \quad Q_0^\Delta=Q_0, \quad C^\Delta=C, \quad \Omega^\Delta=\Omega,
\end{eqnarray}
in addition to the settings in Cor.~\ref{eqi_marg_cor} [Eqs.~\eqref{eq:sde_thmsetting}-~\eqref{eq:k2}], the lower bounds
in Eq.~\eqref{eq:sde_elbo1} and Eq.~\eqref{eq:gpssm_elbo} are equal for $R=1$.
\end{corollary}

\subsection{Proofs}
\label{appx:sde_proofs}
In order to prove Thm.~\ref{sde_theorem}, we require an additional lemma, similar to Lem.~\ref{lemma}:
\begin{lemma} \label{sde_lemma}
	The term $ q^\Delta(f_M,x_j,\dots,x_1|x_0) $ in Eq.~\eqref{eq:sde_joint} can also be written as
	\begin{equation}\label{eq:sde_lemmain}
	q^\Delta(f_M,x_j,\dots,x_1|x_0) = q^\Delta(f_M|x_j,\dots,x_0) \prod_{j'=1}^{j} q^\Delta(x_{j'}|x_{j'-1},\dots,x_0),
	\end{equation}
	for $ j\in \{1,\dots,J\} $. Here, the $ q^\Delta(x_{j}|x_{j-1},\dots,x_0) $ are as in Eqs.~\eqref{eq:sde_thmgauss}-\eqref{eq:sde_thmvar} and
	\begin{align}\label{eq:sde_lemqfM}
	q^\Delta(f_M|x_j,\dots,x_0) &= \gauss{f_M}{\hat{\mu}^j_M}{\hat{\Sigma}^j_M}
\end{align}
with
\begin{align}
	\label{eq:sde_lemmean}
	\hat{\mu}^j_M &=
m^\Delta_M + 
(R \Delta_t) S^\Delta_M \KMinvDelta \left(K^\Delta_{0:j-1,M}\right)^\top \invSt{0:j-1} \left( x_{1:j} - \mut{0:j-1} \right),\\
	\hat{\Sigma}^j_M &= 
	S^\Delta_M - 
	(R \Delta_t)^2
	S^\Delta_M \KMinvDelta \left(K^\Delta_{0:j-1,M}\right)^\top \invSt{0:j-1} K^\Delta_{0:j-1,M} \KMinvDelta S^\Delta_M.
	\end{align}
\end{lemma}
Using this lemma, we can prove Thm.~\ref{sde_theorem} in exactly the same way as we did on p.~\pageref{thmproof} to prove Thm.~\ref{theorem}:
\begin{proof}[\emph{\textbf{Proof of Theorem \ref{sde_theorem}}}]
	\label{sde_thmproof}
	Exactly as on p.~\pageref{thmproof} [Eqs.~\eqref{eq:tproof1} and \eqref{eq:tproof2}], using Eq.~\eqref{eq:sde_marginal} instead of Eq.~\eqref{eq:marginal} and Lem.~\ref{sde_lemma} instead of Lem.~\ref{lemma}.
\end{proof}
This leaves us with the proof of Lem.~\ref{sde_lemma}. Naturally, this proof also works very similarly as the proof of Lem.~\ref{lemma}:
\begin{proof}[\emph{\textbf{Proof of Lemma \ref{sde_lemma}}}]
	\label{sde_lemproof}
	Very similar to the proof in Sec.~\ref{sec:appxproof}.
	Care has to be taken that the factors of $ R\Delta_t $ appearing in Eqs.~\eqref{eq:sde_transition}, \eqref{eq:sde_thmmu}, \eqref{eq:sde_thmS} are treated correctly.
	The only difference apart from this is that the definitions of all quantities from Sec.~\ref{sec:appxgpssm} have to be replaced with their counterparts in Sec.~\ref{sec:appxsde}.
\end{proof}

\begin{proof}[\emph{\textbf{Proof of Corollary \ref{eqi_marg_cor}}}]
 We substitute the quantities from Eqs.~\eqref{eq:sde_thmsetting} -~\eqref{eq:k2}
 into Eqs.~\eqref{eq:sde_qx0} and ~\eqref{eq:sde_thmmain}.
Independently of the resolution $R$, the variational posterior over the initial latent state are consistent, i.e. $q^\Delta(x_0) = q(x_0) $.
For $ R=1 $, we can see, that (recursively) $ q^\Delta(x_j| x_{j-1},\dots,x_0) = q(x_t| x_{t-1},\dots,x_0) $ if $ j = t $ by comparing Eqs.~\eqref{eq:thmgauss} and \eqref{eq:sde_thmgauss}.
This means that the marginals in Eqs.~\eqref{eq:thmmain} and \eqref{eq:sde_thmmain} are equal, implying that the settings in Eq.~\eqref{eq:sde_thmsetting} -~\eqref{eq:k2} satisfy Eq.~\eqref{eq:consistency}.
\end{proof}

\begin{proof}[\emph{\textbf{Proof of Corollary \ref{eqi_elbo_cor}}}]
 We substitute the quantities from Eqs.~\eqref{eq:sde_thmsetting} -~\eqref{eq:k2},~\eqref{eq:sde_elbosetting}
 into Eq.~\eqref{eq:sde_elbo1}.
Independently of the resolution $R$, the two KL-terms are consistent.
For $R = 1$, we have $J = T/R = T$, meaning that we have the same number of summands in the remaining terms. Moreover, using Cor.~\ref{eqi_marg_cor}, we see that for $j=t$ we have $q(x_t) = q^\Delta(x_j)$. This means that the individual summands, $\mathbb{E}_{q^\Delta(x_j)}[p^\Delta(y_t \vert x_t)])$ and $\mathbb{E}_{q(x_t)}[p(y_t \vert x_t)]$, are equal, implying that the sums in  Eqs.~\eqref{eq:sde_elbo1} and~\eqref{eq:gpssm_elbo} are equal. In total, this implies that $\mathcal{L}=\mathcal{L}_\Delta$.
\end{proof}

\newpage

\section{The Prior in Gaussian Process State-Space Models}
\label{sec:appxprior}
\subsection{Analytical Marginalization over the Inducing Outputs}
In this section, we note that we can analytically marginalize over the inducing outputs $f_M$ in the prior [Eq.~\eqref{eq:pjoint}] in a very similar way as we did for the posterior [Eq.~\eqref{eq:q}] in Sec.~\ref{sec:appxgpssm}.
Here, we are interested in $p(x_t)$, the prior marginals of the latent state at time point $ t\in \{1,\dots,T\} $.
They can be obtained as [cf.~Eqs.~\eqref{eq:marginal}-\eqref{eq:transition} and see also the definitions of the variables there]
\begin{equation}\label{eq:marginalprior}
p(x_t) = \int p(x_0) p(f_M,x_t,\dots,x_1|x_0) df_M \prod_{t'=0}^{t-1}dx_{t'},
\end{equation}
where
\begin{align}\label{eq:jointprior}
p(f_M,x_t,\dots,x_1|x_0) &= p(f_M) \prod_{t'=0}^{t-1} p(x_{t'+1}|x_{t'},f_M),\\
\label{eq:px0}
p(x_0) &= \gauss{x_0}{\mu_0}{Q_0},  \\
\label{eq:pfM}
p(f_M) &= \gauss{f_M}{0}{K_{MM}},\\
\label{eq:transitionprior}
p(x_{t+1}|x_{t},f_M) &= \gauss{x_{t+1}}{x_{t} + K_{tM}\KMinv f_M}{Q+K_{tt}-K_{tM}\KMinv K_{tM}^\top }. 
\end{align}

Similarly as in Thm.~\ref{theorem} [see also there for the notation], we find:
\begin{theorem} \label{theoremprior}
	For the prior of the GP SSM as defined above, the marginals of the latent state at time point $ t\in \{1,\dots,T\} $, can be obtained as
	\begin{equation}\label{eq:thmmainprior}
		p(x_t) = \int p(x_0) \left[ \prod_{t'=1}^{t} p(x_{t'}|x_{t'-1},\dots,x_0) \right] \prod_{t'=0}^{t-1} dx_{t'},
	\end{equation}
	where all terms are Gaussian:
	\begin{align} \label{eq:thmgaussprior}
	p(x_{t}|x_{t-1},\dots,x_0) &= \gauss{x_t}{\doublehat{\mu}_t}{\doublehat{\Sigma}_t},\\
	\label{eq:thmmeanprior}
	\doublehat{\mu}_t &= x_{t-1} + \Stt{t-1}{0:t-2} \left( \doubletilde{S}_{0:t-2, 0:t-2}\right)^{-1} \left( x_{1:t-1} - x_{0:t-2} \right),\\
	\label{eq:thmvarprior}
	\doublehat{\Sigma}_t &= \Stt{t-1}{t-1} - \Stt{t-1}{0:t-2} \left( \doubletilde{S}_{0:t-2, 0:t-2}\right)^{-1} \Stt{0:t-2}{t-1}.
	\end{align}
	Here, the terms are given by
	\begin{equation}
	\label{eq:thmSprior}
	\Stt{t}{t'} = K_{tM}\KMinv K_{t'M}^\top + \delta_{tt'}(Q + K_{tt} - K_{tM} \KMinv K_{t'M}^\top).
	\end{equation}
\end{theorem}

Note that we have indexed the variables with double symbols, $\doublehat{\cdot}$ or $\doubletilde{\cdot}$ to clearly distinguish them from the very similar variables for the approximate posterior appearing in Thm.~\ref{theorem}.

\begin{proof}[\emph{\textbf{Proof of Theorem \ref{theoremprior}}}]
The proof is very simple as Thm.~\ref{theoremprior} is a special case of Thm.~\ref{theorem}, where we replace $q(x_0)$ by $p(x_0)$ and $q(f_M)$ by $p(f_M)$ [cf.~Eqs.~\eqref{eq:marginal}-\eqref{eq:transition} and Eqs.~\eqref{eq:marginalprior}-\eqref{eq:transitionprior}].

First, replacing $q(x_0)$ by $p(x_0)$ only has an effect on Eq.~\eqref{eq:thmmain}, where
$q(x_0)$ is replaced by $p(x_0)$ to arrive at Eq.~\eqref{eq:thmmainprior}. This is possible since $q(x_0)$ remains unchanged in the proof of Thm.~\ref{theorem} (see Sec.~\ref{sec:appxproof}).

Second, replacing $q(f_M)$ by $p(f_M)$ amounts to using Thm.~\ref{theorem} in the special case $m_M = 0$ and $S_M = K_{MM}$ [cf.~Eqs.~\eqref{eq:qfM} and \eqref{eq:pfM}].
Plugging these values in Eqs.~\eqref{eq:thmgauss}-\eqref{eq:thmS} leads to the corresponding formulas in Eqs.~\eqref{eq:thmgaussprior}-\eqref{eq:thmSprior}.
\end{proof}

\newpage

\section{Pseudocode}
\label{sec:pseudocode}
Alg.~\ref{alg:backfitting} shows the backfitting algorithm in more detail. The total number of backfitting cycles is denoted by $\tau_c$.
Each iteration consists of two steps.
First, we update the parameters $\theta^{(l)} = \{m_0^{(l)}, S_0^{(l)}, m_M^{(l)}, S_M^{(l)} \}$ of the $l$-th component  by maximizing the lower bound $\mathcal{L}_{\text{}dil}$ (see Alg.~\ref{alg:update}), while keeping the parameters of the remaining components fixed.
During parameter optimization, we only need to simulate the latents $x_T^{(l)}$ of the $l$-th component, while we re-use the cached latents, $x_T^{(\neq l)}$, for all other components.

For simulating trajectories during parameter optimization, we employ the Euler-Maruyama scheme (see Alg.~\ref{alg:simulate}) with stepsize $R^{(l)}$. Here, the use of dilated mini-batches allows us to learn effects on different time scales.
Importantly, we are free to choose a different resolution for computing the cached latents, $x_T^{(\neq l)}$. 
By applying the default resolution level ($R=1$) for the latter, we ensure that the discretization level is sufficiently tight for faster varying dynamics (i.e. $R^{(l')}<R^{(l)}$).

The algorithm has a runtime complexity of $O(\tau_c((M^2LT) + \tau_o(M^2L(B_0+B) + M^3L)))$, where 
$\tau_o$ is the number of optimization steps for updating the parameter of one component (see Alg.~\ref{alg:update}). 

\begin{algorithm}[H]
   \caption{Backfitting Algorithm}
   \label{alg:backfitting}
\begin{algorithmic}
   \State {\bfseries Input:} data $y_T$, initial parameters $\Theta=\{ \theta^{(l)} \}_{l=1}^L$, resolutions $\{R^{(l)}\}_{l=1}^L$, mini-batch size $B$, buffer size $B_0$.
  \For{$i=1$  {\bfseries to} $\tau_c$}
   \For{$l=1$ {\bfseries to} $L$}
\State $\theta^{(l)}$\ \ =  update($y_T$,  $x^{(\neq l)}_T$, $\theta^{(l)}$, $R^{(l)}$, $B$, $B_0$)  \Comment{Alg.~\eqref{alg:update}}
\State $x^{(l)}_T$ = sampleSeq( $\theta^{(l)}$, 1, T, 0) \Comment{Alg.~\eqref{alg:simulate}}
\EndFor 
\EndFor
\State \textbf{return} $\Theta$
\label{algo:backfitting}
\end{algorithmic}
\end{algorithm}

\begin{algorithm}[H]
   \caption{sampleSeq(parameter $ \theta^l$, resolution $R$, mini-batch size $B$, buffer size $B_0$)
   }
   \label{alg:simulate}
\begin{algorithmic}
\State $f^{(l)}_M$ = sampleMultivariateGaussian$(m^{(l)}_M$, $\text{S}^{(l)}_M$) 
\State$x^{(l)}_0$  \ = sampleMultivariateGaussian$(m^{(l)}_0$, $\text{Q}^{(l)}_0$)
   \For{$j=1$ {\bfseries to} $B_0 + B$}
\State $x^{(l)}_{j}$  = sampleNext($x_{j-1}^{(l)}, f_M^{(l)},  R$)  \Comment{Eq.~\eqref{eq:emm}}
\EndFor
\State \textbf{return} $x^{(l)}_{B_0:B_0+B}$
\label{algo:sample}
\end{algorithmic}
\end{algorithm}

\begin{algorithm}[H]
   \caption{update(data $y_T$, latents $x^{(\neq l)}_T$, parameter $\theta^{(l)}$, resolution $ R^{(l)}$, mini-batch size $B$, buffer size $B_0$)}
   \label{alg:update}
\begin{algorithmic}
  \For{$i=1$ {\bfseries to} $\tau_o$}
\State $y_B$ = sampleDilSeq($y_T$,$R^{(l)}$, $B$) \Comment{$y_B$=$\{ y_{j_0+j} \}_{j=1}^B$}
\State $x^{(\neq l)}_B$ = sampleDilSeq($x^{(\neq l)}_T$,$R^{(l)}$, $B$) \Comment{$x^{(\neq l)}_B$=$\{ x^{(\neq l)}_{j_0+j} \}_{j=1}^B$}
\State $x^{(l)}_B$ = sampleSeq($\theta^{(l)}$, $R^{(l)}$, $B$, $B_0$) \Comment{Alg.~\eqref{alg:simulate}}
\State update $\theta^{(l)}$ using $\mathcal{L}_{{\Delta}}$($y_B$, $x_B$)  \Comment{Eq.~\eqref{eq:elbo_sde}}
\EndFor
\State \textbf{return} $\theta^{(l)}$
\label{algo:update}
\end{algorithmic}
\end{algorithm}

\newpage

\section{Additional Results and Experimental Details}
\label{sec:appxexperiments}

\subsection{Parameter Settings}
\label{sec:appx_parameters}
In the following, we give auxiliary details about our experimental protocol.

\begin{itemize}
    \item 
For initialising model and variational parameters, we followed the default values given by \citet{doerr2018probabilistic}. The initial values and additional configuration details are given in Supplementary Table~\ref{table:parameters}. For semi-synthetic experiments we changed the initial observation noise to $0.1^2$. 
\item 
In line with the previous implementation, we restricted the linear emission matrix to be $C^{(l)}:=[\mathbb{I},0]\in \mathbb{R}^{D_y \times D_x}$, where $\mathbb{I} \in \mathbb{R}^{D_y \times D_y}$ is the identity matrix.
\item
We let the initial learning rate decay over time by lowering the learning rate by a multiplicative factor of $0.99$ every $10$ steps.
For MC/MR-GPSSM, we reset the learning rate in each backfitting cycle.
\item
We train each output dimension independently resulting on $4$ models on the semi-synthetic datasets (S, M1, M2, F) and $4$ models on the engine dataset (PN, HC, NOx, Temp).
\item
Our preliminary results suggested that feeding back the Monte Carlo estimates into the backfitting algorithm leads to noisy estimates of the lower bound.
To avoid local optima, we therefore only use the mean partial residual, $\text{mean}({\hat{y}}^{(\neq l)}_T)$, during parameter optimization.
In order to obtain meaningful uncertainties, we opted for a full Monte Carlo treatment during prediction time.
\item 
Empirical runtimes in Supplementary Table~\ref{table:batchsizeVSresolution} are reported on a standard MacBook Pro laptop with 3,1 GHz Dual-Core Intel Core i5 processor.
\end{itemize}

\begin{table}[h]
    \caption{\textbf{Parameter Settings.} Initial values for model and variational parameters (left) and hyperparameter settings (right) applied in all GPSSM/(MC/MR)-GPSSM experiments.}
    \begin{minipage}{.5\linewidth}
        \begin{tabular}{ll}
        \hline \textbf{Parameter} & \textbf{Initialization} \\ \hline
         Inducing inputs& $\boldsymbol{\zeta}_{d} \sim \mathcal{U}(-2,2)$ \\
         Inducing outputs& $q\left(\boldsymbol{z}_{d}\right)=\mathcal{N}\left(\boldsymbol{z}_{d} \mid \boldsymbol{\mu}_{d}, 0.01^{2} \cdot \boldsymbol{I}\right)$ \\
         & $\mu_{d, i} \sim \mathcal{N}\left(\mu_{d, i} \mid 0,0.05^{2}\right)$ \\
         Process noise & $\sigma_{\mathrm{X}}^{2}=0.002^{2}$ \\
        Observation noise & $\sigma_{\mathrm{Y}}^{2}=1^{2}$ \\
         RBF-Kernel & $\sigma^{2}=0.5^{2}$, $lengthscale = 2 $ \\
        \hline
        \end{tabular}
    \end{minipage}%
    \begin{minipage}{.5\linewidth}
      \centering
            \begin{tabular}{ll}
            \hline \textbf{Parameter} & \textbf{Configuration} \\
            \hline \# of inducing points & $50$ \\
            \# of samples & $20$ \\
            \# of minibatches & $20$ \\
            Minibatch size $B$ & $50$ \\
            Buffer size $B_0$ & $10$ \\
            Initial Learning rate & $0.05$ \\
            \hline
            \end{tabular}
    \end{minipage} 
\label{table:parameters}
\end{table}

\newpage

\subsection{Semi-Synthetic Data}
\label{sec:appx:semi_synth}

We demonstrate the need for multiple resolutions using  $4$ semi-synthetic datasets with varying properties: 
slow dynamics (S),
mixed dynamics (M1, M2), and
fast dynamics (F).
All datasets are created based on a single measurement of the engine dataset containing 37,961 datapoints.
The simulations are generated as follows:
\begin{itemize}
\item 
We train on each output one independent standard GPSSM model with one latent state $D_x=1$ each.
For the output HC and NOx, we apply the resolution $R=1$ to extract latent states with fast dynamics, on the outputs Temp and PN we apply the resolution $R=30$ to extract latent states with slow dynamics.
Parameter settings are as described in Supplementary Table~\ref{table:parameters}.
\item 
The semi-synthetic datasets are formed by additively combining the latent states: 
Dataset F (fast dynamics) consists of the sum of the latent states extracted from HC and NOx, Dataset S (slow dynamics) consists of the sum of the latent states extracted from PN and Temp.
Dataset M1 (mixed dynamics) consists of the sum of the latent states extracted from Temp and HC, and dataset M2 analgously from NOx and Temp.
\end{itemize}
The inputs used in the simulation experiments are the original inputs of the engine dataset.
All $4$ datasets are visualized in Supplementary Figure \ref{figapp:simulated_data}. 
We report the predictive performance of the different methods across the four datasets in
Table~\ref{table:synthetic}. Figure~\ref{fig:simulated_oneexample} and
Supplementary Figure~\ref{figapp:simulated_example}
show the predictions on one example run for all  datasets using varying resolutions.
In addition, we investigated if an increased batch size can compensate for using the standard resolution on datasets with slow dynamics in Supplementary
Table~\ref{table:batchsizeVSresolution}.

\begin{table*}[t]
\centering
\caption{
\textbf{Experimental Results on Semi-Synthetic Data using nLL.}
Predictive performance of GPSSM and MR/MC-GPSSM using different resolutions on four semi-synthetic datasets with varying dynamics: slow (S), mixed (M1, M2) and fast (F). The experiment is repeated $5$ times and we report the mean (standard error) over all runs.
The best performing method, and all methods whose mean statistic overlap within the standard error, are marked in bold.}
\begin{small}
\begin{tabular}{lllllllll} 
\hline
\multicolumn{1}{l}{}  & \multicolumn{1}{l}{} & \multicolumn{2}{c}{\textbf{GPSSM }}            &  & \multicolumn{2}{c}{\textbf{MC-GPSSM} }  &  & \multicolumn{1}{c}{\textbf{MR-GPSSM}}  \\ 
\cline{3-4}\cline{6-7}\cline{9-9}
\multicolumn{1}{l}{}  & \multicolumn{1}{l}{} & $R=1$                  & $R=30$                &  & $R=[1,1]$              & $R=[30,30]$    &  & $R=[30, 1]$                            \\ 
\hline
\multirow{4}{*}{nLL}  & \textbf{F}           & \textbf{-33.42 (0.95)} & -15.58 (0.75)         &  & \textbf{-36.02 (3.40)} & -6.29 (6.92)   &  & \textbf{-35.29 (2.58)}                 \\
                      & \textbf{M1}          & -6.26 (6.37)           & -15.96 (0.71)         &  & -11.03 (2.43)          & -13.87 (0.97)  &  & \textbf{-31.49 (2.85)}                 \\
                      & \textbf{M2}          & -13.11 (0.89)          & 50.33 (58.55)         &  & -11.38 (0.57)          & -5.59 (1.06)   &  & \textbf{-27.75 (3.27)}                 \\
                      & \textbf{S}           & $> 1000$               & \textbf{-4.45 (3.82)} &  & 81.34 (45.95)          & 153.65 (73.95) &  & 41.83 (18.37)                          \\
\hline
\end{tabular}
\end{small}
\label{table:synthetic2}
\end{table*}

\begin{landscape}
\begin{table}[h]
\centering
\caption{\textbf{Batch size vs. resolution.} 
We investigate the difference between applying a large resolution ($R=30$) and a small batch size ($B=50$) vs. applying the standard resolution ($R=1$) and a large batch size ($B=1500$).
The experiment is performed on the semi-synthetic dataset S which requires a long history for learning the underlying dynamics.
Our comparison takes the training time into account by running (MC-)GPSSM ($B=1500$, $R=1$) once with the same computational cost but reduced number of iterations, and once with the same number of iterations but increased runtime.
We performed the experiment five times and report the mean RMSE/nLL (standard error) for all variants. 
We exclude runs that have clearly failed due to convergence issues and report the number of these runs in a separate row.
}
\begin{tabular}{lllllllll} 
\hline
                          &  & \multicolumn{2}{c}{\textbf{Large resolution, small batch size }} &  & \multicolumn{4}{c}{\textbf{Standard resolution, large batch size}}  \\ 
\cline{3-4}\cline{6-9}
Batch size                &  & 50            & 50                                               &  & 1500         & 1500           & 1500        & 1500               \\
Iterations                &  & 600           & 12x50                                            &  & 600          & 12x50          & 20          & 12x2               \\
$R$                       &  & 30            & {[}30,30]                                        &  & 1            & {[}1,1]        & 1           & {[}1,1]            \\ 
\hline
\textbf{\# excluded runs} &  & 0             & 0                                                &  & 1            & 1              & 2           & 2                  \\
\textbf{RMSE}             &  & 0.16(0.01)   & 0.20(0.03)                                      &  & 0.30(0.02)  & 0.31(0.01)    & 0.66(0.09) & 0.49(0.09)        \\
\textbf{nLL}              &  & 50.33(58.55) & 153.65(73.95)                                   &  & 51.16(8.98) & 105.50(45.81) & $>1000$     & $>1000$            \\
\textbf{Time (seconds)}   &  & 220           & 585                                              &  & 4958         & 10486          & 201         & 665                \\
\hline
\end{tabular}
\label{table:batchsizeVSresolution}
\end{table}
\end{landscape}

\begin{figure}[h]
    \centering
    \includegraphics[width=\textwidth]{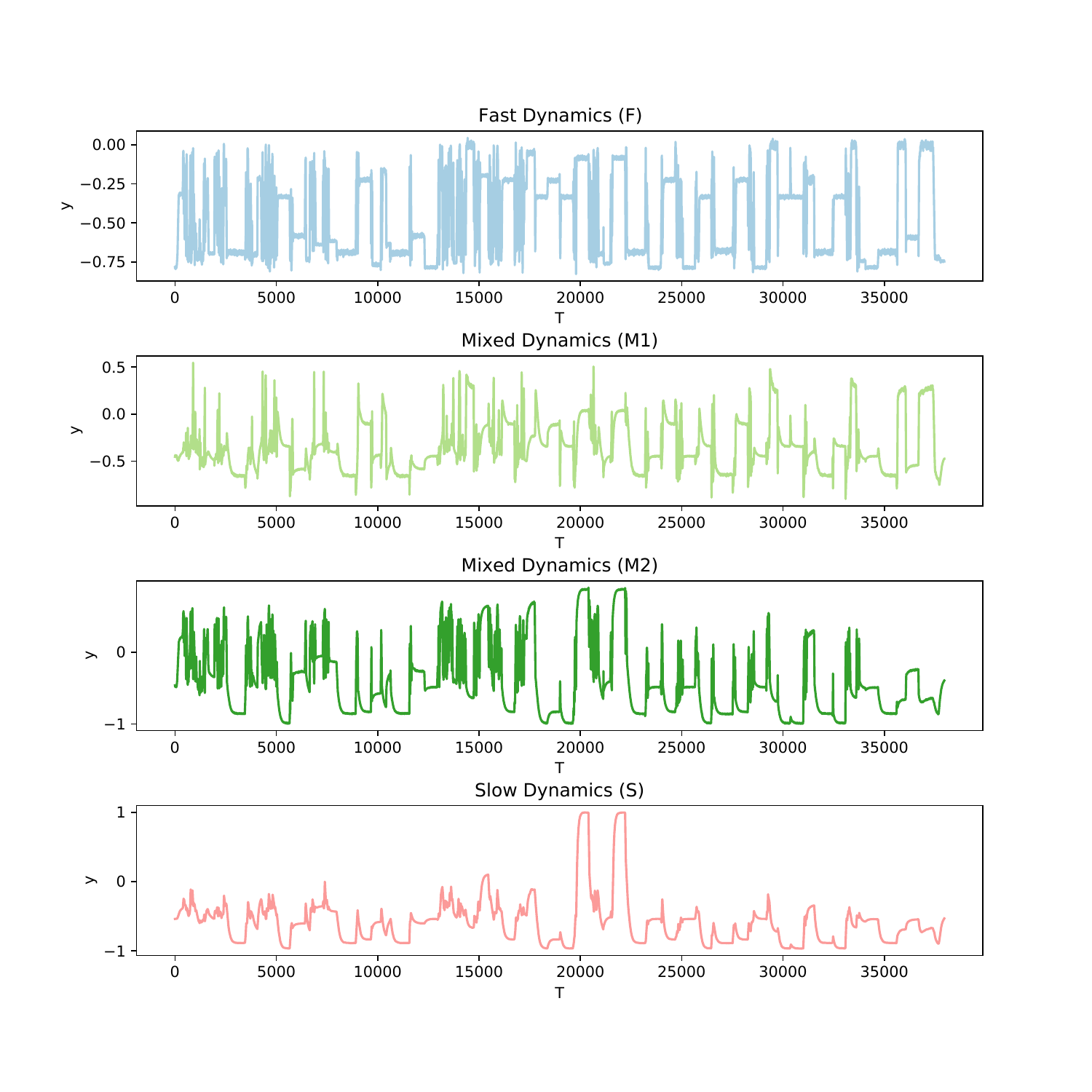}
    \caption{\textbf{Semi-synthetic datasets with varying dynamics.} 
    From top to bottom: one dataset with fast varying dynamics (F), two datasets with mixed dynamics (M1, M2) containing effects on short and long timescales, and one dataset with slowly varying dynamics (S).}
    \label{figapp:simulated_data}
\end{figure}

\begin{figure}[h]
    \centering
    \includegraphics[width=\textwidth]{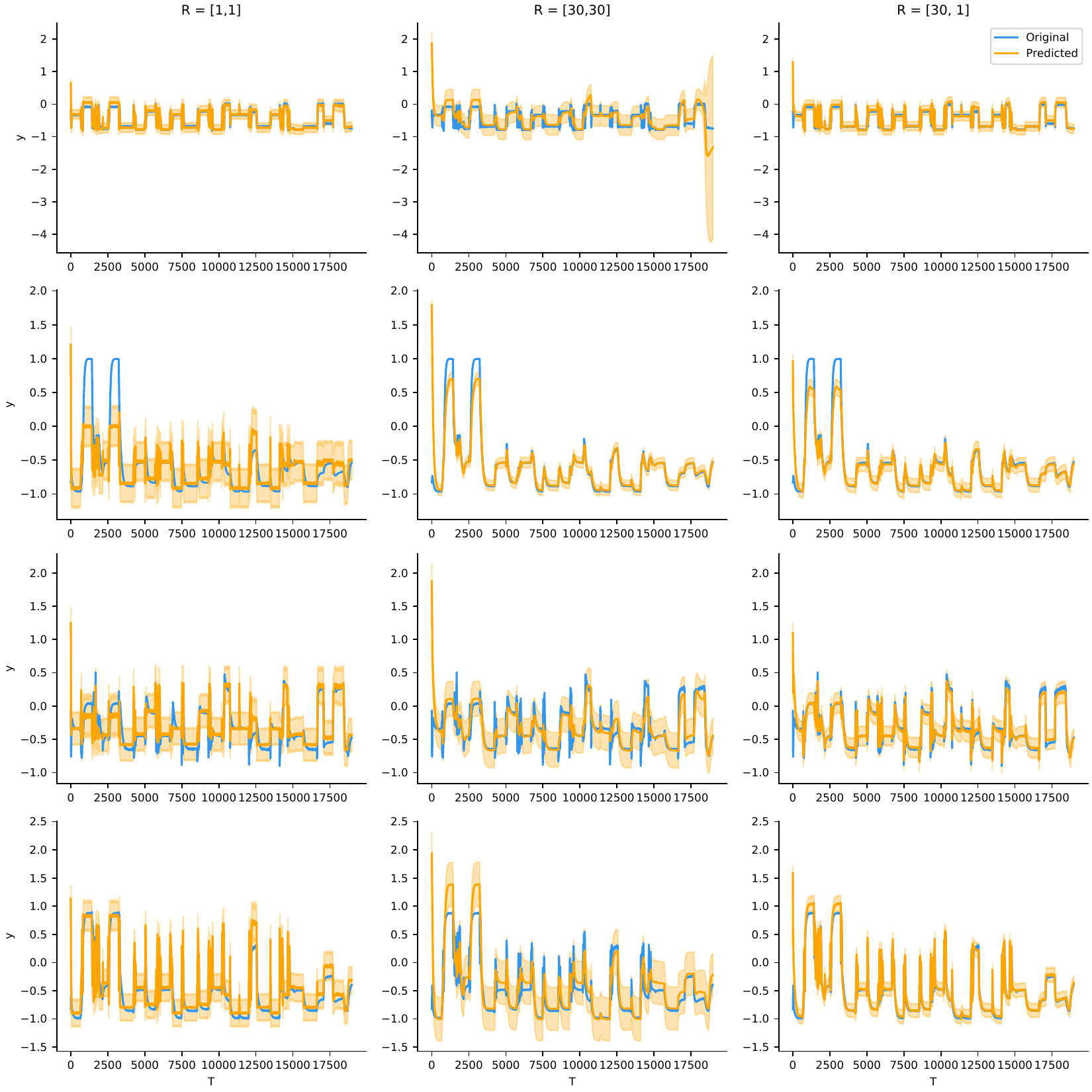}
    \caption{\textbf{Predictions on semi-synthetic dataset with varying dynamics using MC/MR-GPSSM.} 
    From top to bottom:
    dataset F (fast dynamics), 
    dataset S (slow dynamics),
    dataset M1 and M2 (mix of fast and slow dynamics).
    From left to right:
    MC-GPSSM using small resolutions ($R=[1,1]$), MC-GPSSM using large resolutions ($R=[30,30]$), MR-GPSSM using varying resolutions ($R=[1,30]$).
    MC-GPPSM only performs well if the chosen resolution fits well to the data and cannot handle datasets with mixed dynamics.
    Our proposed approach, MR-GPPSM, achieves accurate predictions over all scenarios.
    }
    \label{figapp:simulated_example}
\end{figure}

\newpage 

\subsection{Engine Dataset}
\label{sec:appx_cross_validation}
Capturing the dynamics of an engine is a difficult task requiring a careful design of experiment.
To enable accurate modeling while keeping the measurement costs low, different design strategies were applied across the 21 measurements.
In order to obtain comparable test datasets, we therefore applied stratified cross-validation.
For this, we divided the measurements into the following 6 groups:

\begin{itemize}
    \item 
    Group G0: measurements (mix between standard dynamical design of experiment and test bench drives): 4,8,9 and 16.
    \item 
    Group G1 (low gradients): measurements 10, 11, 17 and 18.
    \item 
    Group G2 (others, not included in test): measurements 3, 7, 12, and 13.
    \item 
    Group G3: (standard dynamical design of experiment is split into four segments according to engine speed and engine torque) measurements 0, 5, 14, 15, 19
    \item 
    Group G4 (standard dynamical design of experiment): measurements 1, 2, 6, 20 
    \item 
    Group G5 (real driving emissions on the road): measurement 21
\end{itemize}
For each experiment, we split the measurements into a training and a test set such that the test split consists of 1 measurement of G0, 2 measurements of G1, 1 measurement of G3, 1 measurement of G4, and the measurement of G5.
We trained for each output (particle numbers, hydrocarbon concentration, nitrogen oxide concentration and engine temperature) a separate model and used the following four inputs: speed, load, lambda and ignition angle.
All results presented on this data have been averaged over five different splits. For each split, we ran the experiment three times and compute the predictions on the test set using the model with maximum lower bound $\mathcal{L}$.
We normalized the inputs and outputs to zero mean and unit variance prior training.

For GPSSM models, we perform 3,000 iterations while for MR/MC-GPSSM, we perform $10$ backfitting loops with $300$ iterations for each component. Table~\ref{table:emission} shows a comparison between the different methods. 
Furthermore, we studied the sensitivity of the results with respect to the chosen resolution set in Supplementary Table~\ref{table:sensitivity}. 

\begin{figure*}[h]
    \centering
    \includegraphics[width=\textwidth]{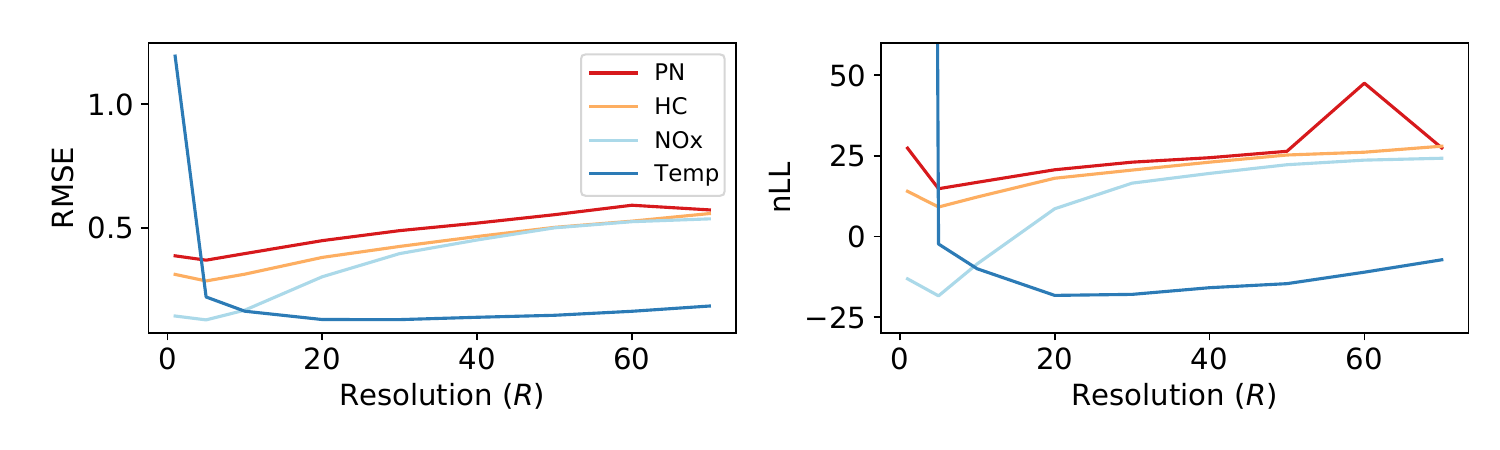}
    \caption{\textbf{Impact of resolution on predictive performance on the engine dataset.
    }
    Shown is the predictive performance (left: RMSE, right: nLL) of GPSSM when varying the resolution $R$ on an independent test set.
    The output Temp requires large resolutions, indicating that the dynamics are slowly varying.
    All other outputs require smaller resolutions, pointing towards fast varying dynamics.
    }
    \label{figapp:gridsearch}
\end{figure*}

\begin{landscape}
\begin{table*}[h]
\centering
\caption{\textbf{Experimental Results on Engine Modeling Task.} Predictive performance of GPSSM and MR/MC-GPSSM and GPSSM emission datasets with four outputs. We report the mean (standard error) over $5$ runs. The best performing method is marked in bold.}

\begin{tabular}{l|l|lllllllll} 
\hline
\multicolumn{1}{l}{}            & \multicolumn{1}{l}{} & \multicolumn{3}{c}{\textbf{GPSSM~}}                             &  & \multicolumn{3}{c}{\textbf{MC-GPSSM }}          &  & \multicolumn{1}{c}{\textbf{MR-GPSSM}}  \\ 
\cline{3-5}\cline{7-9}\cline{11-11}
\multicolumn{1}{l}{}            & \multicolumn{1}{l}{} & $R=1$         & $R=5$                  & $R=30$                 &  & $R=[30,30,30]$ & $R=[5,5,5]$   & $R=[1,1,1]$    &  & $R=[30,5,1]$                           \\ 
\hline
\multirow{4}{*}{\rotcell{RMSE}} & \textbf{PN}          & 0.39 (0.02)   & \textbf{0.37 (0.03)}   & 0.49 (0.03)            &  & 0.48 (0.02)    & 0.40 (0.02)   & 0.42 (0.01)    &  & 0.41 (0.02)                            \\
                                & \textbf{HC}          & 0.31 (0.02)   & \textbf{0.28 (0.01)}   & 0.42 (0.02)            &  & 0.43 (0.03)    & 0.30 (0.02)   & 0.56 (0.12)    &  & 0.32 (0.02)                            \\
                                & \textbf{NOx}         & 0.14 (0.01)   & \textbf{0.13 (0.00)}   & 0.40 (0.01)            &  & 0.41 (0.01)    & 0.14 (0.01)   & 0.19 (0.02)    &  & 0.14 (0.01)                            \\
                                & \textbf{Temp}        & 1.19 (0.21)   & 0.22 (0.00)            & 0.13 (0.00)            &  & 0.12 (0.00)    & 0.21 (0.01)   & 1.16 (0.09)    &  & \textbf{0.11 (0.01)}                   \\ 
\hline
\multirow{4}{*}{\rotcell{nLL}}  & \textbf{PN}          & 27.36(6.92)   & \textbf{14.78 (1.42)}  & 23.02 (1.19)           &  & 22.45 (1.05)   & 16.24 (1.45)  & 26.74 (7.38)   &  & 17.58 (1.00)                           \\
                                & \textbf{HC}          & 13.97 (3.24)  & \textbf{9.12 (0.92)}   & 20.56 (1.19)           &  & 20.53 (1.98)   & 9.92 (1.63)   & 77.05 (36.37)  &  & 11.49 (2.28)                           \\
                                & \textbf{NOx}         & -13.17 (1.99) & \textbf{-18.45 (1.01)} & 16.51 (0.77)           &  & 16.10 (0.68)   & -13.78 (1.88) & 15.52 (15.16)  &  & -14.30 (1.95)                          \\
                                & \textbf{Temp}        & $>1000$       & -2.37 (0.60)           & \textbf{-17.99 (0.53)} &  & -13.64 (1.44)  & -2.11 (1.14)  & 352.99 (33.81) &  & -17.39 (7.15)                          \\
\hline
\end{tabular}
\label{table:emission}
\end{table*}

\begin{table}[h]
\centering
\caption{\textbf{Sensitivity analysis.}
Predictive performance of MR-GPSSM on emission datasets when varying the resolution set. 
The first column $R=[30,5,1]$ corresponds to the default setting.
We can observe that the performance is consistent when varying the size of the resolution set.
The outputs HC and NOx require at least one component that can capture fast dynamics, while the output Temp requires at least one component that can capture slow dynamics.
We report the mean (standard error) over $5$ runs.}
\begin{tabular}{l|l|llllll} 
\hline
\multicolumn{1}{l}{}  & \multicolumn{1}{l}{} & \multicolumn{6}{c}{\textbf{MR-GPSSM}}                                                                           \\ 
\cline{3-8}
\multicolumn{1}{l}{}  & \multicolumn{1}{l}{} & $R=[30,5,1]$  & $R=[7,5,1]$          & $R=[40,20,10]$ & $R=[40,20]$   & $R=[30,15,7,5,1]$ & $R=[40,20,10,5,1]$  \\ 
\hline
\multirow{4}{*}{RMSE} & \textbf{PN}          & 0.41 (0.02)   & 0.40 (0.03)          & 0.41 (0.02)    & 0.45 (0.02)   & 0.40 (0.02)       & 0.39 (0.02)         \\
                      & \textbf{HC}          & 0.32 (0.02)   & 0.29 (0.01)          & 0.38 (0.03)    & 0.40 (0.02)   & 0.31 (0.02)       & 0.32 (0.01)         \\
                      & \textbf{NOx}         & 0.14 (0.01)   & 0.15 ({0.01}) & 0.19 (0.01)    & 0.33 (0.01)   & 0.15 (0.01)       & 0.17 (0.02)         \\
                      & \textbf{Temp}        & 0.11 (0.01)   & 0.22 (0.03)          & 0.10 (0.00)    & 0.11 (0.00)   & 0.10 (0.00)       & 0.10 (0.00)         \\ 
\hline
\multirow{4}{*}{nLL}  & \textbf{PN}          & 17.58 (1.00)  & 20.34 (5.99)         & 17.83 (1.15)   & 20.77 (1.01)  & 16.65 (1.05)      & 16.00 (1.02)        \\
                      & \textbf{HC}          & 11.49 (2.28)  & 10.00 (0.81)         & 16.90 (2.04)   & 18.99 (1.35)  & 12.99 (3.77)      & 13.02 (1.79)        \\
                      & \textbf{NOx}         & -14.3 (1.95)  & -11.79 (1.84)        & -5.91 (0.68)   & 9.72 (0.54)   & -11.76 (2.94)     & 12.39 (22.55)       \\
                      & \textbf{Temp}        & -17.39 (7.15) & 3.67 (10.05)         & -21.94 (2.13)  & -16.69 (0.82) & -26.37 (1.37)     & -26.23 (0.76)       \\
\hline
\end{tabular}
\label{table:sensitivity}
\end{table}
\end{landscape}

\end{document}